\documentclass[11pt]{article}
\usepackage{etoolbox}
\usepackage{covington}
\usepackage{graphicx}
\usepackage{latexsym}
\usepackage{amssymb}
\usepackage{amsmath}
\usepackage{amsthm}
\usepackage{caption}
\usepackage{subcaption}
\usepackage{mwe}
\usepackage{thmtools}
\usepackage{forloop}
\usepackage[paper=letterpaper,margin=1in]{geometry}
\usepackage[ruled,vlined,linesnumbered]{algorithm2e}
\usepackage{paralist}
\usepackage{mleftright}\mleftright
\usepackage{multirow}
\usepackage{array}
\usepackage[dvipsnames]{xcolor}
\usepackage{nicefrac}
\usepackage{hyperref}
\usepackage{pgfplots}
\usepackage{siunitx}
\usepackage{pgfplotstable}
\usepackage{booktabs}
\usepackage{numprint}
\usepackage{pifont}
\usepackage{tikz}
\usepackage[rightcaption]{sidecap}
\usepgfplotslibrary{fillbetween}
\pgfplotsset{compat=1.18}

\makeatletter
\patchcmd{\@algocf@start}% <cmd>
  {-1.5em}% <search>
  {0pt}% <replace>
  {}{}% <success><failure>
\makeatother
 
\definecolor{darkgreen}{rgb}{0,0.5,0}
\hypersetup{
    unicode=false,            % non-Latin characters in Acrobat's bookmarks
    colorlinks=true,          % false: boxed links; true: colored links
    linkcolor=blue!70!black,  % color of internal links (change box color with linkbordercolor)
    citecolor=darkgreen,      % color of links to bibliography
    filecolor=magenta,        % color of file links
    urlcolor=blue!70!black    % color of external links
}

\newtheorem{theorem}{Theorem}[section]
\newtheorem{lemma}[theorem]{Lemma}

\newtheorem{corollary}[theorem]{Corollary}

\newtheorem{proposition}[theorem]{Proposition}
\newtheorem{observation}[theorem]{Observation}

\newcommand{\defcal}[1]{\expandafter\newcommand\csname c#1\endcsname{{\mathcal{#1}}}}
\newcommand{\defbb}[1]{\expandafter\newcommand\csname b#1\endcsname{{\mathbb{#1}}}}
\newcommand{\defvec}[1]{\expandafter\newcommand\csname v#1\endcsname{{\mathbf{#1}}}}
\newcommand{\defmat}[1]{\expandafter\newcommand\csname m#1\endcsname{{\mathbf{#1}}}}
\newcounter{calBbCounter}
\forLoop{1}{26}{calBbCounter}{
    \edef\capitalletter{\Alph{calBbCounter}}
		\edef\letter{\alph{calBbCounter}}
    \expandafter\defcal\capitalletter
		\expandafter\defbb\capitalletter
		\expandafter\defvec\letter
}

\newcommand{\eps}{\varepsilon}
\newcommand{\nnR}{{\bR_{\geq 0}}}
\newcommand{\dotcup}{\mathbin{\mathaccent\cdot\cup}}
\newcommand{\headerref}[1]{{\texorpdfstring{\ref{#1}}{\ref*{#1}}}}
\newcommand{\characteristic}{{\mathbf{1}}}
\newcommand{\RSet}{{\mathsf{R}}}
\newcommand{\SAT}{{\texttt{SAT}}}

\newcommand{\OurAlgorithm}{Min-as-Oracle}
\newcommand{\IterativeAlg}{Iterative-X-Growing}
\newcommand{\SingletonsAlg}{Min-by-Singletons}

\newif\ifcomments
\commentstrue
\ifcomments\newcommand{\comments}[1]{#1}\else\newcommand{\comments}[1]{}\fi

\newcommand{\email}[1]{{\href{mailto:#1}{#1}}}
\title{Submodular Minimax Optimization: Finding Effective Sets}
\author{Loay Mualem\thanks{Computer Science Department, University of Haifa, Israel. E-mail: \email{loaymua@gmail.com}.} \and
				Ethan R. Elenberg\thanks{ASAPP, New York, NY. E-mail: \email{eelenberg@asapp.com}.} \and
				Moran Feldman\thanks{Computer Science Department, University of Haifa, Israel. E-mail: \email{moranfe@cs.haifa.ac.il}.} \and
				Amin Karbasi\thanks{School of Engineering and Applied Science, Yale University, and Google Research, E-mail: \email{amin.karbasi@yale.edu}.}
}

\begin{document}
\pagenumbering{Alph}
\maketitle
\begin{abstract}
Despite the rich existing literature about minimax optimization in continuous settings, only very partial results of this kind have been obtained for combinatorial settings. In this paper, we fill this gap by providing a characterization of submodular minimax optimization, the problem of finding a set (for either the min or the max player) that is effective against every possible response. We show when and under what conditions we can find such sets. We also demonstrate how minimax submodular optimization provides robust solutions for downstream machine learning applications such as (i) efficient prompt engineering for question answering, (ii) prompt engineering for dialog state tracking, (iii) identifying robust waiting locations for ride-sharing, (iv) ride-share difficulty kernelization, and (v) finding adversarial images. Our experiments demonstrate that our proposed algorithms  consistently outperform other baselines. 

\medskip


\end{abstract}
\thispagestyle{empty}
\newpage

\pagenumbering{arabic}
\section{Introduction}

Many machine learning tasks, ranging from data selection to decision making, are inherently combinatorial and thus, require combinatorial optimization techniques that work at scale. Even though, in general, solving such problems is notoriously hard, practical problems are very often endowed with extra structures that lend them to optimization techniques. One common structure is submodularity, a condition that holds either exactly or approximately in a wide range of machine learning applications, including: dictionary selection~\cite{krause2010submodular}, sparse recovery, feature selection~\cite{das2011submodular}, neural network interpretability~\cite{elenberg2017streaming}, crowd teaching~\cite{singla2014near}, human brain mapping~\cite{salehi2017submodular}, data summarizarion~\cite{lin2011class,mualem2022resolving}, among many others. Submodular functions are often considered to be discrete analogs of concave functions, and like concave functions they can be (approximately) maximized. At the same time, submodular functions can also be exactly minimized as they can be extended into a continues convex function (known as the Lov\'{a}sz extension). These optimization properties of submodular functions has been often exploited in scalable machine learning algorithms. 

While scalable optimization methods are desirable, they are not the only requirements for ML algorithm deployment. Very often, it is also important to get solutions that are robust with respect to noise, outliers, adversarial examples, etc. In particular, problems looking for solutions that are robust with respect to worst-case scenarios have usually been expressed as minimax optimization. Accordingly, recent years have witnessed a large body of work addressing minimax optimization in the continuous settings (see, e.g., \cite{diakonikolas2021efficient,ibrahim2020linear,lin2020gradient,mokhtari2020convergance}). This line of research has given rise to a myriad of algorithms, and an ever increasing list of applications such as adversarial attack generation~\cite{wang2021adverserial}, robust statistics~\cite{agarwal2022minimax} and multi-agent systems~\cite{li2019robust}, to name a few. To ensure feasibility of finding a saddle point, one has to make some structural assumptions. For instance, many of the above-mentioned works assume that the minimization is taken over a convex function, and the maximization over a concave function. 

Despite the rich existing literature about minimax optimization in continuous settings, very few works have managed to obtain similar results for combinatorial settings. Staib et al.~\cite{staib2019distributionally} and 
 Adibi et al.~\cite{adibi2022minimax} considered hybrid settings in which the maximization is done with respect to a (discrete) submodular function, but the minimization is still done over a continuous domain. Krause et al.~\cite{krause2008robust}, Torrico et al.~\cite{torrico2021structred} and Iyer~\cite{iyer2019unified,iyer2020robust} considered settings in which both the maximization and the minimization are discrete, but one of them is done over a small domain that can be efficiently enumerated. In this paper we provide the first systematic study of the natural case of fully discrete minimax optimization with maximization and minimization domains that can both be large. To the best of our knowledge, the only previous works relevant to this case are works of Bogunovic et al.~\cite{bogunovic2017robust} and Orlin et al.~\cite{orlin2018robust}. These works studied a particular problem within the general setting we consider. Specifically, they studied the maximization of a monotone submodular function subject to a cardinality constraint in the presence of a worst case (represented by a minimization) removal of a small number of elements from the chosen solution.

As submodular functions cannot be maximized exactly, there is no hope to get a saddle point in our setting. Instead, like Adibi et al.~\cite{adibi2022minimax}, we take a game theoretic perspective on the setting. From this point of view, there are two players. Each player selects a set, and the objective function value is determined by the sets selected by both players. One of the players aims to minimize the objective function, while the other player wishes to maximize it. Our task is to select for one of the players (either the minimization or the maximization player) a set that is \textit{effective} in the sense that it guarantees a good objective value regardless of the set chosen by the other player.

We map the tractability and approximability of the above minimax submodular optimization task as function of various properties, such as: the player considered (minimization or maximization), the constraints (if any) on the sets that can be chosen by the players, and whether the objective function is submodular as a whole, or for each player separately. We refer the reader to Section~\ref{ssc:results} for our exact results (and Sections~\ref{sec:max_min} and~\ref{sec:min_max} for the proofs of these results). However, in a nutshell, we have fully mapped the approximability for the minimization player, and we also have non-trivial results for the maximization player.

Our proposed algorithms for minimax submodular optimization can lead to finding of robust solutions  for down-stream machine learning applications, including efficient prompt engineering, ride-share difficulty kernalization, adversarial attacks on image summarization and robust ride-share optimization.  %The rest of this section is devoted to further discussing the first of these applications. However, 
Empirical evaluation of our algorithms in the context of all the above applications can be found in Section~\ref{sec:exp}.

\subsection{Our theoretical contribution} \label{ssc:results}

Let us describe the formal model for our setting. There are two (disjoint) ground sets $\cN_1$ and $\cN_2$, one ground set for each one of the players. For each ground set $\cN_i$, we also have a constraint $\cF_i \subseteq 2^{\cN_i}$ specifying the sets that can be chosen from this ground set. Finally, there is a non-negative objective set function $f\colon 2^{\cN_1 \dotcup \cN_2} \to \nnR.$\footnote{Here, and throughout the paper, $\dotcup$ denotes the union of disjoint sets.} The minimization player gets to pick a set $X$ from $\cF_1$, and wishes to minimize the value of $f$, while the maximization players picks a set $Y$ from $\cF_2$, and aims to maximize the value of $f$. Our task is to find for each player a set $S$ that yields the best value for $f$ assuming the other player chooses the best response against $S$. In other words, for the minimization player we want to find a set $X$ that (approximately) minimizes
\[
	\min_{X \in \cF_1} \max_{Y \in \cF_2} f(X \dotcup Y)
	\enspace,
\]
and for the maximization player we should find a set $Y$ that (approximately) maximizes 
\[
	\max_{Y \in \cF_2} \min_{X \in \cF_1} f(X \dotcup Y)
	\enspace.
\]
Since optimization of general set functions cannot be done efficiently, we must assume that the objective function $f$ obeys some properties. Two common properties that are often considered in the literature are submodularity and monotonicity. However, to assume these properties, we first need to discuss what they mean in our setting.

Let us begin with the property of submodularity. In the following, given an element $u$ and a set $S$ we use $f(u \mid S) \triangleq f(S \cup \{u\}) - f(S)$ to denote the marginal contribution of the element $u$ to the set $S$.\footnote{Similarly, given two sets $S$ and $T$, we denote by $f(T \mid S) \triangleq f(S \cup T) - f(S)$ the marginal contribution of $T$ to $S$.} According to the standard definition of submodularity,\footnote{A set function $g\colon \cN \to \bR$ is submodular if $g(u \mid S) \geq g(u \mid T)$ for every two sets $S \subseteq T \subseteq \cN$ and element $u \in \cN \setminus T$.} $f$ is submodular if
\[
	f(u \mid S) \geq f(u \mid T)
	\qquad
	\forall\; S \subseteq T \subseteq \cN_1 \dotcup \cN_2, u \in (\cN_1 \dotcup \cN_2) \setminus T
	\enspace.
\]
Since this definition of submodularity treats $\cN_1$ and $\cN_2$ as two parts of one ground set, in the rest of this paper we call a function that obeys it \emph{jointly-submodular}. However, since the ground sets $\cN_1$ and $\cN_2$ play very different roles in our problems, it makes sense to consider also functions that are submodular when restricted to one ground set. For example, we say that $f$ is submodular when restricted to $\cN_1$ if it becomes a submodular function when we fix the set of elements of $\cN_2$ chosen. More formally, $f$ is submodular when restricted to $\cN_1$ if
\[
	f(u \mid S \cup A_2) \geq f(u \mid T \cup A_2)
	\qquad
	\forall\; S \subseteq T \subseteq \cN_1, u \in \cN_1 \setminus T, A_2 \subseteq \cN_2
	\enspace.
\]
The definition of being submodular when restricted to $\cN_2$ is analogous, and we say that $f$ is \emph{disjointly-submodular} if it is submodular when restricted to either $\cN_1$ or $\cN_2$.

Unfortunately, submodular minimization admits very poor approximation guarantees even subject to simple constraints such as cardinality (see Section~\ref{ssc:related_work} for more details). Therefore, we restrict attention to the case of $\cF_1 = 2^{\cN_1}$. Given this restriction, we cannot assume that $f$ is monotone\footnote{A set function $g\colon 2^\cN \to \bR$ is monotone if $g(S) \leq g(T)$ for every two sets $S \subseteq T \subseteq \cN$.} since this will guarantee that the best choice for the set $X$ is always $\varnothing$. However, some of our results assume that $f$ is monotone with respect to the elements of $\cN_2$. In other words, we say that $f$ is \emph{$\cN_2$-monotone} if
\[
	f(u \mid S) \geq 0
	\qquad
	\forall\; S \subseteq \cN_1 \dotcup \cN_2, u \in \cN_2 \setminus S
	\enspace.
\]

Table~\ref{tbl:results} summarizes the theoretical results proved in this paper. When we say in this table that we have an inapproximability result of $c$ for a problem, we mean that no polynomial time algorithm can produce a value that with probability at least $2/3$ approximates the exact value of this problem up to a factor of $c$. For example, if look at the optimization problem $\min_{X \subseteq \cN_1} \max_{Y \in \cF_2} f(X \dotcup Y)$, then an inapproximability result of $c$ means that no polynomial time algorithm can produce a value $v$ which with probability at least $2/3$ obeys
\begin{equation} \label{eq:v}
	\min_{X \subseteq \cN_1} \max_{Y \in \cF_2} f(X \dotcup Y)
	\leq
	v
	\leq
	c \cdot \min_{X \subseteq \cN_1} \max_{Y \in \cF_2} f(X \dotcup Y)
	\enspace.
\end{equation}
In contrast, we hold our algorithms to a higher standard. Specifically, when Table~\ref{tbl:results} states that we have a $c$-approximation algorithm for a problem, it means that the algorithm is able to produce with probability at least $2/3$ two things: a value $v$ of the above kind, and a solution set $S$ for the external $\min$ or $\max$ operation that leads to $c$-approximation when the internal $\min$ or $\max$ is solved to optimality. For example, given the above optimization problem, a $c$-approximation algorithm produces with probability at least $2/3$ both a value $v$ obeying Equation~\eqref{eq:v}, and solution set $S \subseteq \cN_1$ such that
\[
	\min_{X \subseteq \cN_1} \max_{Y \in \cF_2} f(X \dotcup Y)
	\leq
	\max_{Y \in \cF_2} f(S \dotcup Y)
	\leq
	c \cdot \min_{X \subseteq \cN_1} \max_{Y \in \cF_2} f(X \dotcup Y)
	\enspace.
\]
It is important to note that the success probability $2/3$ in the above definitions can always be increased via repetitions. However, such repetitions can usually be avoided since our algorithms are typically either deterministic or naturally have a high success probability.

\begin{table}
\caption{Our theoretical results. We denote by $\alpha$ the approximation ratio that can be obtained for maximizing a non-negative submodular function subject to $\cF_2$. If $f$ happens to be $\cN_2$-monotone, then $\alpha$ can be improved to be the approximation ratio that can be obtained for maximizing a non-negative \emph{monotone} submodular function subject to $\cF_2$.} \label{tbl:results}
\small
\centering
\begin{tabular}{ccc}
\toprule
\textbf{Expression to approximate} & \textbf{Assumptions} & \textbf{Result proved}\\
\midrule
$\max_{Y \in \cF_2} \min_{X \subseteq \cN_1} f(X \dotcup Y)$ & jointly-submodular & $(\alpha + \eps)$-approx. alg. (Thm~\ref{thm:min_submodular})\\
$\max_{Y \subseteq \cN_2} \min_{X \subseteq \cN_1} f(X \dotcup Y)$ & disjointly-submodular & \multicolumn{1}{l}{\multirow{3}{*}{\Bigg\} $\begin{array}{l} \text{No finite approximation ratio}\mspace{-9mu}\\ \text{possible unless $BPP = NP$}\\ \text{(Thms~\ref{thm:inapproximability_unconstrained} and~\ref{thm:inapproximability_cardinality})} \end{array}$}}\\
\multirow{2}{*}{$\max_{\substack{Y \subseteq \cN_2\\|Y| \leq k}} \min_{X \subseteq \cN_1} f(X \dotcup Y)$} & disjointly-submodular & \\
& $\cN_2$-monotone &\\
\midrule
$\min_{X \subseteq \cN_1} \max_{Y \subseteq \cN_2} f(X \dotcup Y)$ & disjointly-submodular & $(4 + \eps)$-approx. alg. (Thm~\ref{thm:min_max_unconstrained})\\
$\min_{X \subseteq \cN_1} \max_{Y \in \cF_2} f(X \dotcup Y)$ & jointly-submodular, $\varnothing \in \cF_2$ & \hspace{-1mm}$O(\alpha\sqrt{|\cN_1|})$-approx. alg. (Thm.~\ref{thm:min_max_sqrt})\hspace{-1mm}\\
\multirow{2}{*}{$\min_{X \subseteq \cN_1} \max_{Y \in \cF_2} f(X \dotcup Y)$} & disjointly-submodular & \multirow{2}{*}{$O(|\cN_2|)$-approx. alg. (Thm~\ref{thm:min_max_disjoint})}\\
& $\{u\} \in \cF_2 \; \forall u \in \cN_2$ &\\
\bottomrule
\end{tabular}
\end{table}

For $\max \min$ expressions (the problem of the maximization player) we have a good understanding of the approximability, and it turns out that this approximability strongly depends on the kind of submodularity guaranteed for $f$. If $f$ is jointly-submodular, then the problem admits roughly the same approximation ratio as the corresponding maximization problem (i.e., the same problem with the $\min$ operation omitted). In contrast, when $f$ is disjointly-submodular, the problem does not admit any finite approximation ratio even in the special cases of: (i) unconstrained maximization, or (ii) maximizing subject to a cardinality constraint on $Y$ of an $\cN_2$-monotone function.\footnote{Note that in the special case of unconstrained maximization of an $\cN_2$-monotone function the problem becomes trivial since it is always optimal to set $Y = \cN_2$.} See Section~\ref{sec:max_min} for formal statements and proofs of these results.

We also have some results for $\min \max$ expressions (the problem of the minimization player), although our understanding of the approximability of such expressions is worse than for $\max \min$ expressions. When $f$ is jointly-submodular and $\cF_2$ obeys some simple properties (that are obeyed, for example, by any down-closed constraint), we get a finite approximation ratio of $O(\min\{\alpha\sqrt{|\cN_1|},\allowbreak |\cN_2|\})$, where $\alpha$ is the approximation ratio that can be obtained for the corresponding maximization problem. Furthermore, when $f$ is disjointly-submodular, we still get $O(|\cN_2|)$-approximation, which improves to $(4 + \eps)$-approximation when both the minimization and the maximization are unconstrained. It is interesting to note the last results are in strict contrast to the situation with $\max \min$ expressions, in which no finite approximation ratio is possible even in the unconstrained case when $f$ is disjointly-submodular. Formal statements and proofs of our results for $\min \max$ expressions can be found in Section~\ref{sec:min_max}.

As is standard in the literature, we assume (throughout the paper) that the access to the objective function $f$ is done via a value oracle that given a set $S \subseteq \cN_1 \dotcup \cN_2$ returns $f(S)$. Furthermore, given a set $S$ and element $u$, we use $S + u$ and $S - u$ to denote $S \cup \{u\}$ and $S \setminus \{u\}$, respectively.

\subsection{Related work} \label{ssc:related_work}

\paragraph{Submodular minimization.} The first polynomial time algorithm for (unconstrained) submodular minimization was obtained by Gr\"{o}tschel et al.~\cite{grotschel1981ellipsoid} using the ellipsoids method. Almost twenty years later, Schrijver~\cite{schrijver2000combinatorial} and Iwata et al.~\cite{iwata2001combinatorial} obtained, independently, the first strongly polynomial time (and combinatorial) algorithms for the problem. Further works have improved over the time complexities of the last algorithms, and the current state-of-the-art algorithm was described by Lee et al.~\cite{lee2015faster} (see also~\cite{axelrod2020near} for a faster approximation algorithm for the problem). %It is also worth mentioning the Fujishige-Wolf algorithm for submodular minimization~\cite{fujishige1980lexicographieally}, which is known to be efficient in practice, although it was only proved to run in pseudo-polymomial time~\cite{chakrabarty2014}.

All the above results apply to unconstrained submodular minimization. Unfortunately, constrained submodular minimization often (provably) admits only very poor approximation guarantees even when the constraint is as simple as a cardinality constraint (see, for example,~\cite{goel2010approximability,svitkina2011submodular}). Nevertheless, there are rare examples of constraints that allow for efficient submodular minimization, such as the constraint requiring the output set to be of even size~\cite{goemans1995minimizing}.

\paragraph{Submodular maximization.} A simple greedy algorithm obtains the optimal approximation ratio of $1 - 1/e$ for maximization of a monotone submodular function subject to a cardinality constraint~\cite{nemhauser1978best,nemhauser1978analysis}. The same approximation ratio was later obtained for general matroid constraints via the continuous greedy algorithm~\cite{calinescu2011maximizing}. %Other works studied ways to accelerate the greedy algorithm, while still guaranteeing $(1 - 1/e - \eps)$-approximation. The state of the art algorithms of this kind are a randomized algorithm running in $O(n \log \eps)$ time~\cite{mirzasoleiman2015lazier}, where $n$ is the size of the ground set, and a deterministic algorithm running in $O(n/\eps)$ time~\cite{kuhnle2021quick,le2022submodular}.
The best possible approximation ratio for unconstrained maximization of a non-monotone submodular function is $1/2$~\cite{feige2011maximizing,buchbinder2015buchbinder}, even for deterministic algorithms~\cite{buchbinder2018deterministic}. However, the approximability of constrained maximization of such functions is not as well understood. Following a long line of works~\cite{buchbinder2014submodular,ene2016constrained,feldman2011unified,lee2009non-monotone,oveisgharan2011submodular,vondrak2013symmetry}, the state-of-the-art algorithm for maximizing a non-monotone submodular function subject to a  cardinality or matroid constraint guarantees $0.385$-approximation~\cite{buchbinder2019constrained}, while the best inapproximability result for these constraints only shows that it is impossible to obtain for them $0.478$-approximation~\cite{gharan2011simulated,qi2022maximizing}.

It is also worth mentioning a line of work~\cite{mirzasoleiman2017deletion,mitrovic2017streaming} aiming to find a small core set such that even if some elements are adversarially chosen for deletion, it is still possible to produce a good solution based on the core set. Note that this line of work differs from the maximization player point of view in our setting, in which the aim is to find a single solution for the maximization player that is good against every choice of the minimization player.

\paragraph{Prompt engineering for natural language processing.}
In-context learning~\cite{dong2022survey} has emerged as a powerful technique to leverage very large language models~\cite{brown2020gpt3,chen2021codex,ouyang2022instructgpt} for Natural Language Processing (NLP) tasks to new tasks without fine-tuning.
Recent works, such as~\cite{min2022rethinking,wang2022supernatural,wei2022chain}, show the importance of crafting good natural language prompts for these models.

Our prompt engineering experiments build on related works which use a neural retrieval model to prompt large language models for open-domain question answering~\cite{si2023reliable} and dialog state tracking~\cite{hu2022incontext}. While these previous works only use the Top-$k$ candidates based on embedding similarity, we formulate a novel combinatorial optimization problem for each application.
 
% Other algorithmic approaches to prompt engineering including prefix-tuning~\cite{li2021prefix}, prompt tuning~\cite{lester2021prompt}, and AutoPrompt~\cite{shin2020autoprompt} learn parameters using gradient-based optimization.
Some works suggested algorithmic approaches to prompt engineering that learn parameters using gradient-based optimization~\cite{lester2021prompt,li2021prefix,shin2020autoprompt,wen2023hardprompts}.
More recently, \cite{zhou2022promptengineers} designed prompts by ranking generations from a secondary language model combined with iterative Monte Carlo search. All of these methods are complex, computationally expensive, and challenging to interpret.

\section{Results for \texorpdfstring{$\max \min$}{maxmin} optimization} \label{sec:max_min}

This section includes our theoretical results for approximation of $\max \min$ expressions. We begin with our positive result, showing that when $f$ is jointly-submodular, the $\min$ operation does not significantly affect the approximability of the problem. The proof of this result is based on the observation that the joint-submodularity of $f$ implies that $\min_{X \subseteq \cN_1} f(X \dotcup Y)$ is a submodular function of $Y$. See Section~\ref{ssc:min_submodular} for details.
\begin{restatable}{theorem}{thmMinSubmodular} \label{thm:min_submodular}
Assume that there exists an $\alpha$-approximation algorithm $ALG$ for the problem of maximizing a non-negative submodular function $g$ subject to $\cF_2$, then there exists a polynomial time algorithm that outputs a set $\hat{Y} \in \cF_2$ and the value $\min_{X \subseteq \cN_1} f(X \dotcup \hat{Y})$, and guarantees that (i) $\min_{X \subseteq \cN_1} f(X \dotcup \hat{Y}) \leq \tau$, and (ii) the expectation of $\min_{X \subseteq \cN_1} f(X \dotcup \hat{Y})$ is at least $\tau / \alpha$, where $\tau = \max_{Y \in \cF_2} \min_{X \subseteq \cN_1} f(X \dotcup Y)$. Furthermore, if $f$ is $\cN_2$-monotone, then it suffices for $ALG$ to obtain $\alpha$-approximation when $g$ is guaranteed to be monotone (in addition to being non-negative and submodular).
\end{restatable}

We would like to note that by assuming in Theorem~\ref{thm:min_submodular} that $ALG$ is an $\alpha$-approximation algorithm, we only mean that the expected value of the solution of $ALG$ is smaller than the value of the optimal solution by at most a factor of $\alpha$. In other words, we do not make any high probability assumption on $ALG$.

If $ALG$ is a deterministic algorithm, then the algorithm whose existence is guaranteed by Theorem~\ref{thm:min_submodular} is also deterministic. However, if $ALG$ is a randomized algorithm, then it might be necessary to use repetitions to get the result stated in the next corollary. Specifically, by a Markov-like argument, the probability that $\min_{X \subseteq \cN_1} f(X \dotcup \hat{Y}) \geq \tau / (\alpha + \eps)$ must be at least $\eps / \alpha^2$, and therefore, by executing the algorithm from Theorem~\ref{thm:min_submodular} $O(\alpha^2 / \eps)$ times, the probability of getting a set $\hat{Y}$ for which $\min_{X \subseteq \cN_1} f(X \dotcup \hat{Y}) \geq \tau / (\alpha + \eps)$ can be made to be at least $2/3$.
\begin{restatable}{corollary}{colMinSubmodular} \label{cor:min_submodular}
Assume that there exists an $\alpha$-approximation algorithm $ALG$ for the problem of maximizing a non-negative submodular function $g$ subject to $\cF_2$. Then, for every polynomially small $\eps \in (0, \alpha]$, there exists a polynomial time algorithm that
(i) outputs a set $\hat{Y} \in \cF_2$ and the value $\min_{X \subseteq \cN_1} f(X \dotcup \hat{Y})$; and
(ii) guarantees that, with probability at least $2/3$, $\min_{X \subseteq \cN_1} f(X \dotcup \hat{Y})$ falls within the range $[\tau/(\alpha + \eps), \tau]$, where $\tau = \max_{Y \in \cF_2} \min_{X \subseteq \cN_1} f(X \dotcup Y)$.
Furthermore, if $f$ is $\cN_2$-monotone, then it suffices for $ALG$ to obtain $\alpha$-approximation when $g$ is guaranteed to be monotone (in addition to being non-negative and submodular).
\end{restatable}

Unfortunately, it turns out that when $f$ is only disjointly submodular, there is very little an algorithm can guarantee. The following two theorems show this for two basic special cases: unconstrained maximization, and maximization subject to a cardinality constraint of an $\cN_2$-monotone function. Both theorems are proved using a reduction showing that the minimization over $X$ can be replaced with a minimization over multiple functions, which allows us to capture well-known NP-hard problems with $\max \min$ expressions that evaluate to $0$ when the correct answer for the NP-hard problem is ``No'', and evaluate to $1$ otherwise. See Section~\ref{ssc:inapproximability} for details.

\begin{restatable}{theorem}{thmInapproximabilityUnconstrained} \label{thm:inapproximability_unconstrained}
When $f$ is only guaranteed to be non-negative and disjointly submodular, no polynomial time algorithm for calculating $\max_{Y \subseteq \cN_2} \min_{X \subseteq \cN_1} f(X \dotcup Y)$ has a finite approximation ratio unless $BPP = NP$.
\end{restatable}

\begin{restatable}{theorem}{thmInapproximabilityCardinality} \label{thm:inapproximability_cardinality}
When $f$ is only guaranteed to be non-negative, $|\cN_2|$-monotone and disjointly submodular, no polynomial time algorithm for calculating $\max_{Y \subseteq \cN_2, |Y| \leq \rho} \min_{X \subseteq \cN_1} f(X \dotcup Y)$, where $\rho$ is a parameter of the problem, has a finite approximation ratio unless $BPP = NP$.
\end{restatable}

\subsection{Proof of Theorem~\headerref{thm:min_submodular}} \label{ssc:min_submodular}

In this section we prove Theorem~\ref{thm:min_submodular}, which we repeat here for convenience.
\thmMinSubmodular*

To prove Theorem~\ref{thm:min_submodular}, let us define, for every set $Y \subseteq \cN_2$, $g(Y) = \min_{X \subseteq \cN_1} f(X \dotcup Y)$. It is well-known that $g$ is a submodular function, and we prove it in the next lemma for completeness (along with additional properties of $g$).
\begin{lemma} \label{lem:g_properties}
The function $g\colon 2^{\cN_2} \to \nnR$ is a non-negative submodular function, and there exists a polynomial time implementation of the value oracle of $g$. Furthermore, if $f$ is $\cN_2$-monotone, then $g$ is monotone (in addition to being non-negative and submodular).
\end{lemma}
\begin{proof}
We begin the proof by considering Algorithm~\ref{alg:value_oracle}. One can observe that this algorithm describes a way to implement a value oracle for $g$ because, by the definitions of $X'$ and $h_Y$,
\[
	f(X' \cup Y)
	=
	h_Y(X')
	=
	\min_{X \subseteq \cN_1} h_Y(X)
	=
	\min_{X \subseteq \cN_1} f(X \cup Y)
	=
	g(Y)
	\enspace.
\]
Furthermore, Algorithm~\ref{alg:value_oracle} can be implemented to run in polynomial time using any polynomial time algorithm for unconstrained submodular minimization because $h_Y$ is a submodular function.
\begin{algorithm}
\caption{Value Oracle Implementation $(Y)$} \label{alg:value_oracle}
	Define $h_Y(X) \triangleq f(X \dotcup Y)$ for every set $X \subseteq \cN_1$.\\
	Find a set $X' \subseteq \cN_1$ minimizing $h_Y(X')$.\\
	\Return $f(X' \dotcup Y)$.
\end{algorithm}

The non-negativity of $g$ follows from the definition of $g$ and the non-negativity of $f$. Proving that $g$ is also submodular is more involved. Let $Y_1$ and $Y_2$ be two arbitrary subsets of $\cN_2$, and let us choose a set $X_i \in \arg \min_{X \subseteq \cN_1} f(X \dotcup Y_i)$ for every $i \in \{1, 2\}$. Then,
\begin{align*}
	g(Y_1) + g(Y_2)
	={} &
	f(X_1 \dotcup Y_1) + f(X_2 \dotcup Y_2)
	\geq
	f((X_1 \cap X_2) \dotcup (Y_1 \cap Y_2)) + f((X_1 \cup X_2) \dotcup (Y_1 \cup Y_2))\\
	\geq{} &
	\min_{X \subseteq \cN_1} f(X \dotcup (Y_1 \cap Y_2)) + \min_{X \subseteq \cN_1} f(X \dotcup (Y_1 \cup Y_2))
	=
	g(Y_1 \cap Y_2) + g(Y_1 \cup Y_2)
	\enspace,
\end{align*}
where the first inequality holds by the submodularity of $f$ since $X_1 \cup X_2 \subseteq \cN_1$ is disjoint from $Y_1 \cup Y_2 \subseteq \cN_2$. This completes the proof that $g$ is submodular.

It remains to prove that $g$ is monotone whenever $f$ is $\cN_2$-monotone. Therefore, in the rest of the proof we assume that $f$ indeed has this property. Then, if the sets $Y_1$ and $Y_2$ obey the inclusion $Y_1 \subseteq Y_2$, then they also obey
\[
	g(Y_2)
	=
	f(X_2 \cup Y_2)
	\geq
	f(X_2 \cup Y_1)
	\geq
	f(X_1 \cup Y_1)
	=
	g(Y_1)
	\enspace,
\]
where the first inequality follows from the $\cN_2$-monotonicity of $f$, and the second inequality follows from the definition of $X_1$.
\end{proof}

We are now ready to prove Theorem~\ref{thm:min_submodular}.
\begin{proof}[Proof of Theorem~\ref{thm:min_submodular}]
Note that Lemma~\ref{lem:g_properties} implies that $g$ has all the properties necessary for $ALG$ to guarantee $\alpha$-approximation for the problem of $\min_{Y \in \cF_2} g(Y)$. Therefore, we can use $ALG$ to implement in polynomial time the procedure described by Algorithm~\ref{alg:min_submodular} (since $ALG$ runs in polynomial time given a polynomial time value oracle implementation for the objective function).
\begin{algorithm}
	\caption{Approximate Using $ALG$} \label{alg:min_submodular}
	Use $ALG$ to get a set $Y' \in \cF_2$ such that $\alpha^{-1} \cdot \max_{Y \in \cF_2} g(Y) \leq \bE[g(Y')] \leq \max_{Y \in \cF_2} g(Y)$.\\
	\Return the set $Y'$ and the value $g(Y')$.
\end{algorithm}

Since the definition of $g$ implies $\max_{Y \in \cF_2} g(Y) = \max_{Y \in \cF_2} \min_{X \subseteq \cN_2} f(X \dotcup Y) = \tau$, the value $g(Y') = \min_{X \subseteq \cN_1} f(X \dotcup Y') \leq \max_{Y \in \cF_2} \min_{X \subseteq \cN_2} f(X \dotcup Y)$ produced by Algorithm~\ref{alg:min_submodular} is at most $\tau$ and in expectation at least $\tau/\alpha$. Therefore, Algorithm~\ref{alg:min_submodular} has all the properties guaranteed by Theorem~\ref{thm:min_submodular}. 
\end{proof}

\subsection{Proofs of Theorems~\headerref{thm:inapproximability_unconstrained} and~\headerref{thm:inapproximability_cardinality}} \label{ssc:inapproximability}

In this section we prove the inapproximability results stated in Theorems~\ref{thm:inapproximability_unconstrained} and~\ref{thm:inapproximability_cardinality}. The proofs of both theorems are based on the reduction described by the following proposition. Below, we use $\bN_0$ and $\bN$ to denote the set of natural numbers with and without $0$, respectively. Additionally, recall that for a non-negative integer $i$, $[i] = \{1, 2, \dotsc, i\}$. In particular, this implies that $[0] = \varnothing$, which is a property we employ later in the section.
\begin{proposition} \label{prop:reduction}
Fix any family $F_2$ of pairs of ground set $\cN_2$ and constraint $\cF_2 \subseteq 2^{\cN_2}$. Additionally, let $\alpha \colon \bN_0 \times F_2 \to [1, \infty)$ be an arbitrary function (intuitively, for every pair $(\cN_2, \cF_2) \in F_2$, $\alpha(m, \cN_2, \cF_2)$ is an approximation ratio that we assign to this pair when the ground set $\cN_1$ has a size of $m$). Assume that there exists a (possibly randomized) polynomial time algorithm $ALG$ which, given a ground set $\cN_1$, a pair $(\cN_2, \cF_2) \in F_2$, and a non-negative disjointly submodular function $f \colon 2^{\cN_1 \dotcup \cN_2} \to \nnR$, outputs a value $v$ such that, with probability at least 2/3,
\[
	\tfrac{1} {\alpha(|\cN_1|, (\cN_2, \cF_2))} \cdot \max_{Y \in \cF_2} \min_{X \subseteq \cN_1} f(X \dotcup Y)
	\leq
	v
	\leq
	\max_{Y \in \cF_2} \min_{X \subseteq \cN_1} f(X \dotcup Y)
	\enspace.
\]
Then, there also exists a polynomial time algorithm that given a pair $(\cN_2, \cF_2) \in F_2$ and non-negative submodular functions $g_1, g_2, \dotsc, g_m \colon 2^{\cN_2} \to \nnR$ outputs a value $v$ such that, with probability at least 2/3,
\[
	\tfrac{1}{\alpha(m - 1, (\cN_2, \cF_2))} \cdot \max_{Y \in \cF_2} \min_{1 \leq i \leq m} g_i(Y)
	\leq
	v
	\leq
	\max_{Y \in \cF_2} \min_{1 \leq i \leq m} g_i(Y)
	\enspace.
\]
Furthermore, if the functions $g_1, g_2, \dotsc, g_m \colon 2^{\cN_2} \to \nnR$ are all guaranteed to be monotone (in addition to being non-negative and submodular), then it suffices for $ALG$ to have the above guarantee only when $f$ is $\cN_2$-monotone (in addition to being non-negative and disjointly submodular).
\end{proposition}

Before proving Proposition~\ref{prop:reduction}, let us show that it indeed implies Theorems~\ref{thm:inapproximability_unconstrained} and~\ref{thm:inapproximability_cardinality}.
\thmInapproximabilityUnconstrained*
\begin{proof}
Fix the family $F_2 = \{([k], 2^{[k]}) \mid k \in \bN\}$. Assume that there exists a polynomial time algorithm for calculating $\max_{Y \subseteq \cN_2} \min_{X \subseteq \cN_1} f(X \dotcup Y)$ that has a polynomial approximation ratio. By plugging this algorithm and the family $F_2$ into Proposition~\ref{prop:reduction}, we get that there exists a polynomial time algorithm $ALG$ and a polynomial function $\alpha \colon \bN \times \bN \to [1, \infty)$ such that, given integer $k \in \bN$ and $m$ non-negative monotone submodular functions $g_1, g_2, \dotsc, g_m$, the algorithm $ALG$ produces a value $v$ such that, with probability at least 2/3,
\[
	\tfrac{1}{\alpha(m, k)} \cdot \max_{Y \subseteq [k]} \min_{1 \leq i \leq m} g_i(Y)
	\leq
	v
	\leq
	\max_{Y \subseteq [k]} \min_{1 \leq i \leq m} g_i(Y)
	\enspace.
\]
In particular, the algorithm $ALG$ answers correctly with probability at least $2/3$ whether the expression $\max_{Y \subseteq [k]} \min_{1 \leq i \leq m} g_i(Y)$ is equal to zero. Therefore, to prove the theorem it suffices to show that that exists some NP-hard problem such that every instance $I$ of this problem can be encoded in polynomial time as an expression of the form $\max_{Y \subseteq [k]} \min_{1 \leq i \leq m} g_i(Y)$ that takes the value $0$ if and only if the correct answer for the instance $I$ is ``No''.

In the rest of this proof, we show that this is indeed the case for the NP-hard problem {\SAT}. Every instance of {\SAT} consists a CNF formula $\phi$ over $n$ variables $x_1, x_2, \dotsc, x_n$ that has $\ell$ clauses. To encode this instance, we need to construct $n + \ell$ functions over the ground set $[2n]$. Intuitively, for every integer $1 \leq i \leq n$ the elements $2i - 1$ and $2i$ of the ground set correspond to the variable $x_i$ of $\phi$. The element $2i - 1$ corresponds to an assignment of $1$ to this variable, and the element $2i$ corresponds to an assignment of $0$. For every integer $1 \leq i \leq n$, the objective of the function $g_i$ is to make sure that exactly one value is assigned to $x_i$. Formally, this is done by defining $g_i(Y) \triangleq |\{2i - 1, 2i\} \cap Y| \bmod 2$ for every $Y \subseteq [2n]$. One can note that $g_i(Y)$ takes the value $1$ only when exactly one of the elements $2i - 1$ or $2i$ belongs to $Y$. Furthermore, one can verify that $g_i$ is non-negative and submodular.

Next, we need to define the functions $g_{n + 1}, g_{n + 2}, \dotsc, g_{n + \ell}$. To define these functions, let us denote by $c_1, c_2, \dotsc, c_\ell$ the clauses of $\phi$. Additionally we denote by $c_j(x_i = v)$ an indicator that gets the value $1$ if assigning the value $v$ to $x_i$ guarantees that the clause $c_j$ is satisfied. In other words, $c_j(x_i = v)$ equals $1$ only if $v = 1$ and $c_j$ includes the positive literal $x_i$, or $v = 0$ and $c_j$ includes the negative literal $\bar{x}_j$. For every integer $1 \leq j \leq \ell$, the function $g_{n + j}(Y)$ corresponds to the clause $w_j$ and takes the value $1$ only when this clause is satisfied by some element of $Y$. Formally,
\[
	g_{n + j}(Y)
	=
	\max_{i \in Y} c_j(x_{\lceil i/2 \rceil} = (i \bmod 2))
\]
(notice that $x_{\lceil i/2 \rceil}$ is the index of the variable corresponding to element $i$, and $i \bmod 2$ is the value assigned to this variable by the element $i$). One can verify that $g_{n + j}(Y)$ is a non-negative submodular (and even monotone) function.

Let us now explain why $\max_{Y \subseteq [2n]} \min_{1 \leq i \leq m} g_i(Y)$ takes the value $0$ if and only if $\phi$ is not satisfiable. First, if there exists a satisfying assignment $a$ for $\phi$, then one can construct a set $Y \subseteq [2n]$ that encodes $a$. Specifically, for every integer $1 \leq i \leq n$, $Y$ should include $2i - 1$ (and not $2i$) if $a$ assigns the value $1$ to $x_i$, and otherwise $Y$ should include $2i$ (and not $2i - 1$). Such a choice of $Y$ will make all the above functions $g_1, g_2, \dotsc, g_{n + \ell}$ take the value $1$, and therefore, $\max_{Y \subseteq [2n]} \min_{1 \leq i \leq m} g_i(Y) = 1$ in this case. Consider now the case in which $\phi$ does not have a satisfying assignment. Then, for every set $Y \subseteq [2n]$ we must have one of the following. The first option is that $Y$ includes either both $2i - 1$ and $2i$, or neither of these elements, for some integer $i$, which makes $g_i$ evaluate to $0$ on $Y$. The other option is that $Y$ corresponds to some legal assignment $a$ of values to $x_1, x_2, \dotsc, x_n$ that violates some clause $c_j$, and thus, $g_{n + j}$ evaluates to $0$ on $Y$. In both cases $\min_{1 \leq i \leq m} g_i(Y) = 0$.
\end{proof}

\thmInapproximabilityCardinality*
\begin{proof}
The proof of this theorem is very similar to the proof of Theorem~\ref{thm:inapproximability_unconstrained}, and therefore, we only describe here the differences between the two proofs. First, the family $F_2$ should be chosen this time as $\cF_2 = \{([2k], \{Y \subseteq [2k] \mid |Y| \leq k\} \mid k \in \bN\}$. This modification implies that we now need to encode $\phi$ as an instance of
\[
	\max_{\substack{Y \subseteq [2n]\\|Y| \leq n}} \min_{1 \leq i \leq m} g_i(Y)
	\enspace,
\]
where the functions $g_i(Y)$ are all non-negative monotone submodular functions over the ground set $[2n]$. We do this using $n + \ell$ functions like in the proof of Theorem~\ref{thm:inapproximability_unconstrained}. Moreover, the functions $g_{n + 1}, g_{n + 2}, g_{n + \ell}$ are defined exactly like in the proof of Theorem~\ref{thm:inapproximability_unconstrained}.

For every integer $1 \leq i \leq n$, the function $g_i$ still corresponds to the variable $x_i$, but now the role of $g_i$ is only to guarantee that $x_i$ gets at least a single value. This is done by setting $g_i(Y) = \min\{|Y \cap \{2i - 1, 2i\}|, 1\}$, which means that $g_i$ takes the value $1$ only when at least one of the elements $2i - 1$ or $2i$ belongs to $Y$. Note that $g_i$ is indeed non-negative, monotone and submodular, as necessary. The main observation that we need to make is that if $Y$ is a set of size at most $n$ for which all the functions $g_1, g_2, \dotsc, g_n$ return $1$, then $Y$ must include at least one element of the pair $\{2i - 1, 2i\}$ for every integer $1 \leq i \leq n$. Since these are $n$ disjoint pairs, and $Y$ contains at most $n$ elements, we get that $Y$ contains \emph{exactly} one element of each one of the pairs $\{2i - 1, 2i\}$. In other words, $\min_{1 \leq i \leq n} g_i(Y) = 1$ if and only if $Y$ corresponds to assigning exactly one value to every variable $x_i$, which is exactly the property that the functions $g_1, g_2, \dotsc, g_n$ need to have to allow the rest of the proof of Theorem~\ref{thm:inapproximability_unconstrained} to go through.
\end{proof}

\noindent \textbf{Remark:}
The above proof of Theorem~\ref{thm:inapproximability_cardinality} plugs into Proposition~\ref{prop:reduction} the observation that an expression of the form $\max_{Y \subseteq \cN_2, |Y| \leq k} \min_{1 \leq i \leq m} g_i(Y)$ can capture an NP-hard problem. The last observation was already shown by Theorem~3 of Krause et al.~\cite{krause2008robust} (for the Hitting-Set problem). Thus, Theorem~\ref{thm:inapproximability_cardinality} can also be obtained as a corollary of Proposition~\ref{prop:reduction} and Theorem~3 of~\cite{krause2008robust}. However, for completeness and consistency, we chose to provide a different proof of Theorem~\ref{thm:inapproximability_cardinality} that closely follows the proof of Theorem~\ref{thm:inapproximability_unconstrained}.

We now get to the proof of Proposition~\ref{prop:reduction}. One can observe that to prove this proposition it suffices to show the following lemma (the algorithm whose existence is guaranteed by Proposition~\ref{prop:reduction} can be obtained by simply applying $ALG$ to the ground set $\cN_1$ and function $f$ defined by Lemma~\ref{lem:f_construction}).
\begin{lemma} \label{lem:f_construction}
Given non-negative submodular functions $g_1, g_2, \dotsc, g_m \colon 2^{\cN_2} \to \nnR$, there exists a ground set $\cN_1$ and a non-negative disjointly submodular function $f\colon 2^{\cN_1 \dotcup \cN_2} \to \nnR$ such that
\begin{itemize}
	\item the size of the ground set $\cN_1$ is $m - 1$.
	\item given sets $X \subseteq \cN_1$ and $Y \subseteq \cN_2$, it is possible to evaluate $f(X \dotcup Y)$ in polynomial time.
	\item for every set $Y \subseteq \cN_2$, $\min_{X \subseteq \cN_1} f(X \dotcup Y) = \min_{1 \leq i \leq m} g_i(Y)$.
	\item when the functions $g_1, g_2, \dotsc, g_m$ are all monotone (in addition to being non-negative and submodular), then the function $f$ is guaranteed to be $\cN_2$-monotone (in addition to being non-negative and disjointly submodular).
\end{itemize}
\end{lemma}

The rest of this section is devoted to proving Lemma~\ref{lem:f_construction}. Let us start by describing how the ground set $\cN_1$ and the function $f$ are constructed. We assume without loss of generality that $\cN_2 \cap [m - 1] = \varnothing$, which allows us to choose $\cN_1 = [m - 1]$. Given a set $X \subseteq [m - 1]$, let us define $c(X) \triangleq \max\{i \in \bN_0 \mid [i] \subseteq X\}$ (in other words, $c(X)$ is the largest integer such that all the numbers $1$ to $i$ appear in $X$). Additionally, we choose $M$ to be a number obeying $g_i(Y) \leq M/2$ for every $i \in [m]$ and $Y \subseteq \cN_2$ (such a number can be obtained in polynomial time by running the $2$-approximation algorithm of Buchbinder \& Feldman~\cite{buchbinder2018deterministic} for unconstrained submodular maximization on the functions $g_1, g_2, \dotsc, g_m$, and then setting $M$ to be four times the largest number returned). Using this notation, we can now define, for every two sets $X \subseteq \cN_1$ and $Y \subseteq \cN_2$,
\[
	f(X \dotcup Y)
	\triangleq
	g_{c(X) + 1}(Y) + (|X| - c(X)) \cdot M
	\enspace.
\]

The following observation states some properties of $f$ that immediately follow from the definition of $f$ and the fact that $c(X)$ is at most $|X|$ by definition.
\begin{observation} \label{obs:f_properties}
The function $f$ is non-negative and can be evaluated in polynomial time. Furthermore, $f$ is $\cN_2$-monotone when the functions $g_1, g_2, \dotsc, g_m$ are monotone because $g_{c(X) + 1}(Y) + (|X| - c(X)) \cdot M$ is a monotone function of $Y$ for any fixed set $X \subseteq \cN_1$.
\end{observation}

The following two lemmata prove additional properties of $f$.

\begin{lemma} \label{lem:f_submodular}
The function $f$ is disjointly submodular, i.e., it is submodular when restricted to either $\cN_1$ or $\cN_2$.
\end{lemma}
\begin{proof}
For every fixed set $X \subseteq \cN_1$, there exists a value $i \in [m]$ and another value $c$, both depending only on $X$, such that $f(X \dotcup Y) = g_i(Y) + c$. Since adding a constant to a submodular function does not affect its submodularity, this implies that $f$ is submodular when restricted to $\cN_2$. In the rest of the proof we concentrate on showing that $f$ is also submodular when restricted to $\cN_1$.

Consider now an arbitrary element $i \in \cN_1$. For every two sets $X \subseteq \cN_1 - i$ and $Y \subseteq \cN_2$,
\[
	f(i \mid X \dotcup Y)
	=
	g_{c(X + i) + 1}(Y) - g_{c(X) + 1}(Y) + (1 + c(X) - c(X + i)) \cdot M \enspace.
\]
To show that $f$ is submodular when restricted to $\cN_1$, we need to show that the last expression is a down-monotone function $X$, i.e., that its value does not increase when elements are added to $X$. To do that, it suffices to show that the addition to $X$ of any single element $j \in \cN_1 \setminus (X + i)$ does not increase the value of this expression; which we show below by considering a few cases.

The first case we need to consider is the case of $[i - 1] \not \subseteq X + j$. Clearly, in this case $c(X) = c(X + i)$ and $c(X + j + i) = c(X + j)$, and therefore,
\begin{align*}
	f(i \mid (X + j) \dotcup Y)
	={} &
	g_{c(X + j + i) + 1}(Y) - g_{c(X + j) + 1}(Y) + (1 + c(X + j) - c(X + j + i)) \cdot M
	=
	M\\
	={} &
	g_{c(X + i) + 1}(Y) - g_{c(X) + 1}(Y) + (1 + c(X) - c(X + i)) \cdot M
	=
	f(i \mid X \dotcup Y)
	\enspace.
\end{align*}

The second case we consider the case in which $[i - 1] \subseteq X + j$, but $[i - 1] \not \subseteq X$. In this case
\begin{align*}
	f(i \mid (X + j) \dotcup Y)
	={} &
	g_{c(X + j + i) + 1}(Y) - g_{c(X + j) + 1}(Y) + (1 + c(X + j) - c(X + j + i)) \cdot M\\
	\leq{} &
	g_{c(X + j + i) + 1}(Y) - g_{c(X + j) + 1}(Y)
	\leq
	g_{c(X + i) + 1}(Y) - g_{c(X) + 1}(Y) + M\\
	={} &
	g_{c(X + i) + 1}(Y) - g_{c(X) + 1}(Y) + (1 + c(X) - c(X + i)) \cdot M
	=
	f(i \mid X \dotcup Y)
	\enspace,
\end{align*}
where the first inequality holds since the definition of the case implies $c(X + j + i) \geq i = 1 + c(X + j)$, the second inequality follows from the definition of $M$, and the penultimate equality holds since the definition of the case implies $c(X) = c(X + i)$.

The third case we need to consider is when $[i - 1] \subseteq X$ and $c(X + i) = c(X + i + j)$. Since we also have in this case $c(X) = i - 1 = c(X + j)$, we get
\begin{align*}
	f(i \mid (X + j) \dotcup Y)
	={} &
	g_{c(X + j + i) + 1}(Y) - g_{c(X + j) + 1}(Y) + (1 + c(X + j) - c(X + j + i)) \cdot M\\
	={} &
	g_{c(X + i) + 1}(Y) - g_{c(X) + 1}(Y) + (1 + c(X) - c(X + i)) \cdot M
	=
	f(i \mid X \dotcup Y)
\end{align*}

The last case we need to consider is when $[i - 1] \subseteq X$ and $c(X + i) < c(X + j + i)$. In this case
\begin{align*}
	f(i \mid (X + j) \dotcup Y)
	={} &
	g_{c(X + j + i) + 1}(Y) - g_{c(X + j) + 1}(Y) + (1 + c(X + j) - c(X + j + i)) \cdot M\\
	\leq{} &
	g_{c(X + j + i) + 1}(Y) - g_{c(X + j) + 1}(Y) - M
	\leq
	g_{c(X + i) + 1}(Y) - g_{c(X) + 1}(Y)\\
	={} &
	g_{c(X + i) + 1}(Y) - g_{c(X) + 1}(Y) + (1 + c(X) - c(X + i)) \cdot M
	=
	f(i \mid X \dotcup Y)
	\enspace,
\end{align*}
where the first inequality holds since the definition of the case implies $c(X + j) = i - 1 = c(X + i) - 1 < c(X + j + i) - 1$, the second inequality follows from the definition of $M$, and the penultimate equality holds since the definition of the case implies $c(X) = i - 1 = c(X + i) - 1$.
\end{proof}

\begin{lemma} \label{lem:min_g}
For every set $Y \subseteq \cN_2$, $\min_{X \subseteq \cN_1} f(X \dotcup Y) = \min_{1 \leq i \leq m} g_i(Y)$.
\end{lemma}
\begin{proof}
Observe that for every integer $1 \leq i \leq m$, we have $f([i - 1] \dotcup Y) = g_i(Y)$ because $|[i - 1]| = c([i - 1]) = i - 1$. Therefore,
\begin{equation} \label{eq:min_getting}
	\min_{1 \leq i \leq m} f([i - 1] \dotcup Y)
	=
	\min_{1 \leq i \leq m} g_i(Y)
	\enspace.
\end{equation}

Consider now an arbitrary subset $X$ of $\cN_1$ that is not equal to $[i - 1]$ for any integer $1 \leq i \leq m$. For such a subset we must have $c(X) \leq |X| - 1$, and therefore,
\[
	f(X \dotcup Y)
	=
	g_{c(X) + 1}(Y) +	(|X| - c(X)) \cdot M
	\geq
	g_{c(X) + 1}(Y) +	M
	\geq
	M
	\geq
	\min_{1 \leq i \leq m} g_i(Y)
	\enspace,
\]
where the second inequality follows from the non-negativity of $g_{c(X) + 1}$, and the last inequality holds by the definition of $M$. Combining this inequality with Equation~\eqref{eq:min_getting} completes the proof of the lemma.% since
%\[
	%\min_{X \subseteq \cN_1} f(X \dotcup Y)
	%=
	%\min\left\{\min_{1 \leq i \leq m} f([i - 1] \dotcup Y), \min_{\substack{X \subseteq \cN_1 \\ \forall_{i \in [m]} X \neq [i - 1]}} \mspace{-18mu} f(X \dotcup Y)\right\}
	%\enspace.
%\]
\end{proof}

Lemma~\ref{lem:f_construction} now follows by combining Observation~\ref{obs:f_properties}, Lemma~\ref{lem:f_submodular} and Lemma~\ref{lem:min_g}.
\section{Results for \texorpdfstring{$\min \max$}{minmax} optimization} \label{sec:min_max}

This section includes our theoretical results for approximation of $\min \max$ expressions. We begin with the following theorem, which shows that when all the singleton subsets of $\cN_1$ are feasible choices for the $\min$ operation it is possible to get a finite approximation (specifically, $O(|\cN_2|)$-approximation) for $\min \max$. The proof of Theorem~\ref{thm:min_max_disjoint} can be found in Section~\ref{sec:min_max_disjoint}. In a nutshell, it is based on the observation for a set $X \subseteq \cN_1$ the sum $f(X) + \sum_{u \in \cN_2} f(X \dotcup \{u\})$ is an easy to calculate submodular function of $X$ that gives $O(|\cN_2|)$-approximation for $\max_{Y \subseteq \cN_2} f(X \dotcup Y)$.
\begin{restatable}{theorem}{thmMinMaxDisjoint} \label{thm:min_max_disjoint}
Assuming $\{u\} \in \cF_2$ for every $u \in \cN_2$, there exists a polynomial time algorithm that, given a non-negative disjointly submodular function $f\colon 2^\cN \to \nnR$, returns a set $\hat{X} \subseteq \cN_1$ and a value $v$ such that both $\max_{Y \in \cF_2} f(\hat{X} \dotcup Y)$ and $v$ fall within the range $[\tau, (|\cN_2| + 1) \cdot \tau]$, where $\tau \triangleq \min_{X \subseteq \cN_1} \max_{Y \in \cF_2} f(X \dotcup Y)$. 
\end{restatable}

The approximation ratio of the last theorem can be improved to a constant when the $\max$ operation is unconstrained (like the $\min$ operation). The following theorem states this formally, and its proof can be found in Section~\ref{sec:min_max_unconstrained}. This proof is based on using samples of $\cN_2$ to construct a random easy to calculate submodular function of $X$ that with high probability approximates $\max_{Y \subseteq \cN_2} f(X \dotcup Y)$ up to a factor of roughly $4$.
\begin{restatable}{theorem}{thmMinMaxUnconstrained} \label{thm:min_max_unconstrained}
For every constant $\eps \in (0, 1)$, there exists a polynomial time algorithm that given a non-negative disjointly submodular function $f\colon 2^\cN \to \nnR$ returns a set $\hat{X} \subseteq \cN_1$ and a value $v$ such that the expectations of both $\max_{Y \subseteq \cN_2} f(\hat{X} \dotcup Y)$ and $v$ fall within the range $[\tau, (4 + \eps)\tau]$, where $\tau \triangleq \min_{X \subseteq \cN_1} \max_{Y \subseteq \cN_2} f(X \dotcup Y)$. Furthermore, the probability that both $\max_{Y \subseteq \cN_2} f(\hat{X} \dotcup Y)$ and $v$ fall within this range is at least $1 - O(|\cN_2|^{-1})$.
\end{restatable}
The factor of $4 + \eps$ in the last theorem can be improved to $2 + \eps$ when $f$ is symmetric with respect to $\cN_2$, i.e., when $f(X \dotcup Y) = f(X \dotcup (\cN_2 \setminus Y))$ for every two sets $X \subseteq \cN_1$ and $Y \subseteq \cN_2$.

The last theorem in this section obtains a sub-linear approximation guarantee for an (almost) general constraint $\cF_2$; however, this comes at the cost of requiring $f$ to be jointly-submodular. Unlike in the previous theorems of the section, the algorithm used to prove Theorem~\ref{thm:min_max_sqrt} does not use an easy to calculate approximation for $\max_{Y \in \cF_2} f(X \dotcup Y)$. Instead, it grows an output set $X$ in iterations. In each iteration, the algorithm uses the given oracle to find a set $Y \in \cF_2$ whose marginal contribution with respect to the current solution $X$ is large, and then it finds a set $X'$ of elements whose addition to the solution makes the marginal contribution of $Y$ small. After enough iterations, this process is guaranteed to produce a solution $X$ such that every set $Y \in \cF_2$ has a small marginal contribution with respect to $X$. However, such a set $X$ is a good solution only if $f(X)$ is also small. Thus, when finding the augmentation $X'$ it is important to balance two objectives: making the marginal contribution of $Y$ small, and keeping the increase $f(X' \mid X)$ in the value of $X$ small (these objectives are captured by the two terms in the minimized expression on Line~\ref{line:prime_selection} of Algorithm~\ref{alg:square_root_approximation}). See Section~\ref{sec:min_max_sqrt} for the formal proof of Theorem~\ref{thm:min_max_sqrt}.

\begin{restatable}{theorem}{thmMinMaxSqrt} \label{thm:min_max_sqrt}
Assuming $\varnothing \in \cF_2$, there exists a polynomial time algorithm that gets as input
\begin{enumerate}[(i)]
	\item a non-negative jointly submodular function $f\colon 2^\cN \to \nnR$, and
	\item an oracle that given a set $X \subseteq \cN_1$ returns a set $Y \in \cF_2$ that maximizes $f(X \dotcup Y)$ up to a factor of $\alpha \geq 1$ among such sets,
\end{enumerate}
and given this input returns a set $\hat{X} \subseteq \cN_1$ and a value $v$ such that both $\max_{Y \in \cF_2} f(\hat{X} \dotcup Y)$ and $v$ are lower bounded by $\tau$ and upper bounded by $O(\alpha\sqrt{|\cN_1|}) \cdot \tau$, where $\tau \triangleq \min_{X \subseteq \cN_1} \max_{Y \in \cF_2} f(X \dotcup Y)$.
\end{restatable}

Theorem~\ref{thm:min_max_sqrt} assumes an oracle that never fails. Such an oracle can be implemented by a deterministic $\alpha$-approximation algorithm, or a randomized algorithm that maximizes $f(X \dotcup Y)$ up to a factor of $\alpha$ with high probability (in the later case, the algorithm guaranteed by the theorem also succeeds only with high probability). If only a randomized $\alpha$-approximation algorithm is available, then repetitions should be used to get an oracle that maximizes $f(X \dotcup Y)$ up to a factor of $\alpha + \eps$ with high probability. Note that when $\eps > 0$ is only polynomially small, this requires only a polynomial number of repetitions since we may assume that $\alpha \leq |\cN_2|$ (otherwise, Theorem~\ref{thm:min_max_disjoint} already provides a better approximation).

The $\sqrt{|\cN_1|}$ term in Theorem~\ref{thm:min_max_sqrt} might give the impression that this theorem is related to the work of Goemans et al.~\cite{goemans2009approximating} on approximation of submodular functions using linear ones. However, we note that the work of Goemans et al.~\cite{goemans2009approximating} applies only to monotone submodular functions, and in our case the objective function $f$ is never assumed to be monotone (although it is sometimes assumed to be $\cN_2$-monotone).

\subsection{Proof of Theorem~\headerref{thm:min_max_disjoint}} \label{sec:min_max_disjoint}

In this section we prove Theorem~\ref{thm:min_max_disjoint}, which we repeat here for convenience.
\thmMinMaxDisjoint*

The algorithm that we use to prove Theorem~\ref{thm:min_max_disjoint} is given as Algorithm~\ref{alg:min_max_n_2}. We note that the function $g$ defined by this algorithm is the average of $|\cN_2| + 1$ submodular functions (since $f$ is submodular once the subset of $\cN_2$ in the argument set is fixed), and therefore, $g$ is also submodular. As written, Algorithm~\ref{alg:min_max_n_2} is good only for the case in which $\varnothing \in \cF_2$, and for simplicity, we assume throughout the section that this is indeed the case. If $\varnothing \not \in \cF_2$, then the term $f(X)$ should be dropped from the definition of $g$ in Algorithm~\ref{alg:min_max_n_2}, which allows the proof to go through.

\begin{algorithm}
\caption{Estimating the Min-Max using Singletons} \label{alg:min_max_n_2}
Define a function $g \colon 2^{\cN_1} \to \nnR$ as follows. For every set $X \subseteq \cN_1$, $g(X) \triangleq f(X) + \sum_{u \in \cN_2} f(X \dotcup \{u\})$.\\
Use an unconstrained submodular minimization algorithm to find $X' \subseteq \cN_1$ minimizing $g(X')$.\\
\Return the set $X'$ and the value $g(X')$.
\end{algorithm}

The analysis of Algorithm~\ref{alg:min_max_n_2} is based on the observation that $g(X)$ provides an approximation for $\max_{Y \in \cF_2} f(X \dotcup Y)$.
\begin{lemma} \label{lem:g_approximate_n_2}
For every set $X \subseteq \cN_1$, $\max_{Y \in \cF_2} f(X \dotcup Y) \leq g(X) \leq (|\cN_2| + 1) \cdot \max_{Y \in \cF_2} f(X \dotcup Y)$.
\end{lemma}
\begin{proof}
Let $Y'$ be the set in $\cF_2$ maximizing $f(X \dotcup Y')$, then the disjoint submodularity of $f$ guarantees that
\begin{align*}
	\max_{Y \in \cF_2} f(X \dotcup Y)
	={} &
	f(X \cup Y')
	\leq
	f(X) + \sum_{u \in Y'} f(u \mid X)\\
	\leq{} &
	f(X) + \sum_{u \in Y'} f(X \dotcup \{u\})
	\leq
	f(X) + \sum_{u \in \cN_2} f(X \dotcup \{u\})
	=
	g(X)
	\enspace,
\end{align*}
where the second and last inequalities hold by the non-negativity of $f$. This completes the proof of the first inequality of the lemma. To see why the other inequality holds as well, we note that $g(X)$ is the sum of $|\cN_2| + 1$ terms, each of which is individually upper bounded by $\max_{Y \in \cF_2} f(X \dotcup Y)$.
\end{proof}

Using the last lemma, we can now prove Theorem~\ref{thm:min_max_disjoint}.
\begin{proof}[Proof of Theorem~\ref{thm:min_max_disjoint}]
Let $X^*$ be the set minimizing $\max_{Y \in \cF_2} f(X^* \dotcup Y)$. Then, by Lemma~\ref{lem:g_approximate_n_2} and the choice of $X'$ by Algorithm~\ref{alg:min_max_n_2},
\begin{align*}
	\tau
	={} &
	\min_{X \subseteq \cN_1} \max_{Y \in \cF_2} f(X \dotcup Y)
	\leq
	\max_{Y \in \cF_2} f(X' \dotcup Y)
	\leq
	g(X')
	\leq
	g(X^*)\\
	\leq{} &
	(|\cN_2| + 1) \cdot \max_{Y \in \cF_2} f(X^* \dotcup Y)
	=
	(|\cN_2| + 1) \cdot \min_{X \subseteq \cN_1} \max_{Y \in \cF_2} f(X \dotcup Y)
	=
	(|\cN_2| + 1)\tau
	\enspace.
	\qedhere
\end{align*}
\end{proof}

\subsection{Proof of Theorem~\headerref{thm:min_max_unconstrained}} \label{sec:min_max_unconstrained}

In this section we would like to prove Theorem~\ref{thm:min_max_unconstrained}. However, the majority of the section is devoted to proving the following slightly different theorem, which implies Theorem~\ref{thm:min_max_unconstrained}.
\begin{theorem} \label{thm:min_max_unconstrained_fail}
For every constant $\eps \in (0, 1/2)$, there exists a polynomial time algorithm that given a non-negative disjointly submodular function $f\colon 2^\cN \to \nnR$ returns a set $\hat{X}$ and a value $v$ such that
\begin{itemize}
	\item the expectation of $v$ falls within the range $[\tau, (4 + \eps/2)\tau]$, where $\tau \triangleq \min_{X \subseteq \cN_1} \max_{Y \subseteq \cN_2} f(X \dotcup Y)$, and
	\item with probability at least $1 - \eps / [8(|\cN_2| + 1)]$, both $v$ and $\max_{Y \subseteq \cN_2} f(\hat{X} \dotcup Y)$ fall within the range $[\tau, (4 + \eps/2)\tau]$.
\end{itemize}
\end{theorem}

Before getting to the proof of Theorem~\ref{thm:min_max_unconstrained_fail}, let us show that it indeed implies Theorem~\ref{thm:min_max_unconstrained}, which we repeat here for convenience.
\thmMinMaxUnconstrained*
\begin{proof}
Since the guarantee of Theorems~\ref{thm:min_max_unconstrained} becomes stronger as $\eps$ becomes smaller, it suffices to prove the theorem for $\eps \in (0, 1/2)$. Furthermore, the only way in which the algorithm guaranteed by Theorem~\ref{thm:min_max_unconstrained_fail} might not obey the properties described in Theorem~\ref{thm:min_max_unconstrained} is if the expectation of $\max_{Y \subseteq \cN_2} f(\hat{X} \dotcup Y)$ for its output set $\hat{X}$ does not fall within the range $[\tau, (4 + \eps)\tau]$. Thus, to prove Theorem~\ref{thm:min_max_unconstrained} it is only necessary to show how to modify the output set $\hat{X}$ of Theorem~\ref{thm:min_max_unconstrained_fail} in a way that does not violate the other properties guaranteed by this theorem, but makes the expectation of $\max_{Y \subseteq \cN_2} f(\hat{X} \dotcup Y)$ fall into the right range. We do that using Algorithm~\ref{alg:fail_fix}. This algorithm uses a deterministic polynomial time algorithm that obtains $2$-approximation for unconstrained submodular maximization. Such an algorithm was given by Buchbinder and Feldman~\cite{buchbinder2018deterministic}.
\begin{algorithm}
\caption{Best of Two $(\eps)$} \label{alg:fail_fix}
\DontPrintSemicolon
Execute the algorithm guaranteed by Theorem~\ref{thm:min_max_unconstrained_fail}. Let $X'$ denote its output set.\\
Use an algorithm for unconstrained submodular maximization to find a set $Y' \subseteq \cN_2$ such that $\max_{Y \subseteq \cN_2} f(X' \dotcup Y) \leq 2f(X' \dotcup Y') \leq 2 \cdot \max_{Y \subseteq \cN_2} f(X' \dotcup Y)$.

\BlankLine

Execute the algorithm guaranteed by Theorem~\ref{thm:min_max_disjoint}. Let $X''$ denote its output set.\\
Use an algorithm for unconstrained submodular maximization to find a set $Y'' \subseteq \cN_2$ such that $\max_{Y \subseteq \cN_2} f(X'' \dotcup Y) \leq 2f(X'' \dotcup Y'') \leq 2 \cdot \max_{Y \subseteq \cN_2} f(X'' \dotcup Y)$.

\BlankLine

\lIf{$f(X' \dotcup Y') \leq 2f(X'' \dotcup Y'')$}{\Return $X'$.}
\lElse{\Return $X''$.}
\end{algorithm}

Let us denote the output set of Algorithm~\ref{alg:fail_fix} by $\hat{X}$, and observe that the choice of the output set in the last two lines of Algorithm~\ref{alg:fail_fix} guarantees that whenever $\hat{X} = X''$, we also have
\[
	\max_{Y \subseteq \cN_2} f(\hat{X} \dotcup Y)
	=
	\max_{Y \subseteq \cN_2} f(X'' \dotcup Y)
	\leq
	2f(X'' \dotcup Y'')
	\leq
	f(X' \dotcup Y')
	\leq
	\max_{Y \subseteq \cN_2} f(X' \dotcup Y)
	\enspace.
\]
Since the inequality $\max_{Y \subseteq \cN_2} f(\hat{X} \dotcup Y) \leq \max_{Y \subseteq \cN_2} f(X' \dotcup Y)$ trivially applies also when $\hat{X} = X'$, we get that this inequality always hold, and therefore, with probability at least $1 - \eps / [8(|\cN_2| + 1)]$ we must have
\[
	\max_{Y \subseteq \cN_2} f(\hat{X} \dotcup Y)
	\leq
	(4 + \eps/2)\tau
\]
because Theorem~\ref{thm:min_max_unconstrained_fail} guarantees that this inequality holds with at least this probability when $\hat{X}$ is replaced with $X'$. Furthermore, since we always have $\tau = \min_{X \subseteq \cN_1} \max_{Y \subseteq \cN_2} f(X \dotcup Y) \leq \max_{Y \subseteq \cN_2} f(\hat{X} \dotcup Y)$, the inequality $\max_{Y \subseteq \cN_2} f(\hat{X} \dotcup Y) \leq \max_{Y \subseteq \cN_2} f(X' \dotcup Y)$ also shows that $\hat{X}$ falls within the range $[\tau, (4 + \eps)\tau]$ whenever $X'$ falls within the this range. 

Next, we need to prove a second upper bound on $\max_{Y \subseteq \cN_2} f(\hat{X} \dotcup Y)$. By the choice of the output set in the last two lines of Algorithm~\ref{alg:fail_fix}, when this output set is $X'$, we have
\[
	\max_{Y \subseteq \cN_2} f(\hat{X} \dotcup Y)
	=
	\max_{Y \subseteq \cN_2} f(X' \dotcup Y)
	\leq
	2f(X' \dotcup Y')
	\leq
	4f(X'' \dotcup Y'')
	\leq
	4 \cdot \max_{Y \subseteq \cN_2} f(X'' \dotcup Y)
	\enspace.
\]
Since the non-negativity of $f$ implies that the inequality $\max_{Y \subseteq \cN_2} f(\hat{X} \dotcup Y) \leq 4 \cdot \max_{Y \subseteq \cN_2} f(X'' \dotcup Y)$ applies also when $\hat{X} = X''$, we get by Theorem~\ref{thm:min_max_disjoint},
\[
	\max_{Y \subseteq \cN_2} f(\hat{X} \dotcup Y)
	\leq
	4 \cdot \max_{Y \subseteq \cN_2} f(X'' \dotcup Y)
	\leq
	4(|\cN_2| + 1)\tau
	\enspace.
\]

We are now ready to prove that the expectation of the expression $\max_{Y \subseteq \cN_2} f(\hat{X} \dotcup Y)$ falls within the range $[\tau, (4 + \eps)\tau]$ as is guaranteed by Theorem~\ref{thm:min_max_unconstrained}. The expectation is at least the lower end of this range because, as mentioned above, it always holds that $\tau = \min_{X \subseteq \cN_1} \max_{Y \subseteq \cN_2} f(X \dotcup Y) \leq \max_{Y \subseteq \cN_2} f(\hat{X} \dotcup Y)$. Additionally, by the law of total expectation and the two above proved upper bounds on $\max_{Y \subseteq \cN_2} f(\hat{X} \dotcup Y)$,
\begin{align*}
	\bE\left[\max_{Y \subseteq \cN_2} f(\hat{X} \dotcup Y)\right]
	\leq{} &
	\left(1 - \frac{\eps}{8(|\cN_2| + 1)}\right) \cdot \left(4 + \frac{\eps}{2}\right)\tau + \frac{\eps}{8(|\cN_2| + 1)} \cdot 4(|\cN_2| + 1)\tau\\
	={} &
	\left[\left(4 + \frac{\eps}{2}\right) + \frac{\eps}{2}\right]\tau
	=
	(4 + \eps)\tau
	\enspace.
	\qedhere
\end{align*}
\end{proof}

It remains to prove Theorem~\ref{thm:min_max_unconstrained_fail}. The algorithm that we use for this purpose is given as Algorithm~\ref{alg:min_max_4_algorithm}. We note that the function $g$ defined by this algorithm is the average of $m$ submodular functions (since $f$ is submodular once the subset of $\cN_2$ in the argument set is fixed), and therefore, $g$ is also submodular.
\begin{algorithm}
\caption{Estimating the Min-Max via Random Subsets $(\eps)$} \label{alg:min_max_4_algorithm}
Let $n_1 = |\cN_1|$ and $n_2 = |\cN_2|$, and pick $m = \lceil 3200\eps^{-2} [(n_1 + 1) \ln 2 + \ln (n_2 + 1) + \ln(8 / \eps)] \rceil$ uniformly random (and independent) subsets $Y_1, Y_2, \dotsc, Y_m$ of $\cN_2$.\\
Define a function $g \colon 2^{\cN_1} \to \nnR$ as follows. For every $X \subseteq \cN_1$, $g(X) \triangleq \tfrac{1}{m} \sum_{i = 1}^m f(X \dotcup Y_i)$.\\
Use an unconstrained submodular minimization algorithm to find a set $X' \subseteq \cN_1$ minimizing $g(X')$.\\
\Return the set $X'$ and the value $(4 + \eps/2) \cdot g(X')$.
\end{algorithm}

The analysis of Algorithm~\ref{alg:min_max_4_algorithm} uses the following known lemma.
\begin{lemma}[Lemma~2.2 of~\cite{feige2011maximizing}] \label{lem:sampling}
Given a submodular function $f\colon 2^\cN \to \nnR$ and two sets $A, B \subseteq \cN$, if $A(p)$ and $B(q)$ are independent random subsets of $A$ and $B$, respectively, such that $A(p)$ includes every element of $A$ with probability $p$ (not necessarily independently), and $B(q)$ includes every element of $B$ with probability $q$ (again, not necessarily independently), then
\[
	\bE[f(A(p) \cup B(q))]
	\geq
	(1 - p)(1 - q) \cdot f(\varnothing) + p(1 - q) \cdot f(A) + (1 - p)q \cdot f(B) + pq \cdot f(A \cup B)
	\enspace.
\]
\end{lemma}

Given a vector $\vx \in [0, 1]^{\cN_2}$, we define $\RSet(\vx)$ to be a random subset of $\cN_2$ that includes every element $u \in \cN_2$ with probability $x_u$, independently. Given a set $S \subseteq \cN_2$, it will also be useful to denote by $\characteristic_S$ the characteristic vector of $S$, i.e., the vector in $[0, 1]^{\cN_2}$ that has $1$ in the coordinates corresponding to the elements of $S$, and $0$ in the other coordinates. Using this notation, we can now define a function $h\colon 2^{\cN_1} \to \nnR$ as follows. For every set $X \subseteq \cN_1$, $h(X) \triangleq \bE[f(X \dotcup \RSet(\nicefrac{1}{2} \cdot \characteristic_{\cN_2}))]$.
The following lemma shows that $h(X)$ is related to $\max_{Y \subseteq \cN_2} f(X \dotcup Y)$.
\begin{lemma} \label{lem:g_approx_h}
For every set $X \subseteq \cN_1$,
$
	\max_{Y \subseteq \cN_2} f(X \dotcup Y) \leq 4  h(X) %\leq 4 \cdot \max_{Y \subseteq \cN_2} f(X \dotcup Y)
	%\enspace.
$.
\end{lemma}
\begin{proof}
%The second inequality of the lemma holds trivially since $h(X)$ was defined as the expectation of $f(X \dotcup \RSet(\nicefrac{1}{2} \cdot \characteristic_{\cN_2}))$ and $\RSet(\nicefrac{1}{2} \cdot \characteristic_{\cN_2})$ is always a subset of $\cN_2$. Therefore, we concentrate on proving the first inequality of the lemma.
Let us denote by $Y'(X)$ an arbitrary set in $\arg \max_{Y \subseteq \cN_2} f(X \dotcup Y)$, and define $r_X(Y) \triangleq f(X \cup Y)$. Then,
\begin{align*}
	h(X)
	={} &
	\bE[f(X \dotcup \RSet(\nicefrac{1}{2} \cdot \characteristic_{\cN_2})]
	=
	\bE[r_X(\RSet(\nicefrac{1}{2} \cdot \characteristic_{Y'(X)}) \cup \RSet(\nicefrac{1}{2} \cdot \characteristic_{\cN_2 \setminus Y'(X)})]\\
	\geq{} &
	\tfrac{1}{4} r_X(Y'(X))
	=
	\tfrac{1}{4} f(X \dotcup Y'(X))
	=
	\tfrac{1}{4} \cdot \max_{Y \subseteq \cN_2} f(X \dotcup Y)
	\enspace,
\end{align*}
where the inequality follows from Lemma~\ref{lem:sampling} and the observation that for any fixed set $X \subseteq \cN_1$ the function $r_X$ is a non-negative submodular function.
\end{proof}

The last lemma shows that the function $h$ is useful. The following lemma complements the picture by showing that the function $g$ defined by Algorithm~\ref{alg:min_max_4_algorithm} is a good approximation of $h$.
\begin{lemma} \label{lem:surrogate_function}
With probability at least $1 - \eps / [8(n_2 + 1)]$, for every set $X \subseteq \cN_1$ (at the same time) we have $|g(X) - h(X)| \leq (\eps/20) \cdot h(X)$.
\end{lemma}
\begin{proof}
Fix some set $X \subseteq \cN_1$, and let us define $Z_i \triangleq f(X \dotcup Y_i)$ for every integer $1 \leq i \leq m$. We would like to study the properties of the random variables $Z_1, Z_2, \dotsc, Z_m$. First, note that these random variables are independent since the sets $Y_1, Y_2, \dotsc, Y_m$ are chosen independently by Algorithm~\ref{alg:min_max_4_algorithm}. Second, by the definition of $h$ and the distribution of $Y_i$, $\bE[Z_i] = \bE[f(X \dotcup Y_i)] = h(X)$. We would also like to bound the range of values that the random variables $Z_1, Z_2, \dotsc, Z_m$ can take. On the one hand, these random variables are non-negative since $f$ is non-negative. On the other hand, $Z_i = f(X \dotcup Y_i) \leq \max_{Y \subseteq \cN_2} f(X \dotcup Y) \leq 4h(X)$, where the second inequality holds by Lemma~\ref{lem:g_approx_h}.
%
%if we denote by $U$ the subset of $\cN_2$ that maximizes $f(X \dotcup U)$, then by Lemma~\ref{lem:sampling} and the non-negativity of $f$, for every integer $1 \leq i \leq m$,
%\[
	%\bE[Z_i]
	%=
	%\bE[f(X \dotcup Y_i)]
	%=
	%\bE[f(X \cup \RSet(\nicefrac{1}{2} \cdot \characteristic_U) \cup \RSet(\nicefrac{1}{2} \cdot \characteristic_{\cN_2 \setminus U}))]
	%\geq
	%\tfrac{1}{4}f(X \dotcup U)
	%\enspace.
%\]
%Since $f(X \dotcup U)$ always upper bounds $Z_i$ by the definition $U$, the last inequality implies that the maximum value that the variable $Z_i$ can take is $f(X \dotcup U) \leq 4\bE[Z_i] = 4h(X)$.

Given the above proved properties of the random variables $Z_1, Z_2, \dotsc, Z_m$, Hoeffding's inequality shows that
\begin{align*}
	\Pr[|g(X) - h(X)&| \leq (\eps/20) \cdot h(X)]
	=
	\Pr\mleft[\left|\tfrac{1}{m} \sum_{i = 1}^m Z_i - \bE\mleft[\tfrac{1}{m} \sum_{i = 1}^m Z_i\mright]\right| > (\eps/20) \cdot h(X)\mright]\\
	\leq{} &
	2e^{-\frac{2m^2[(\eps/20) \cdot h(X)]^2}{\sum_{i = 1}^m [4h(X)]^2}}
	=
	2e^{-\frac{m\eps^2}{3200}}
	\leq
	2e^{-(n_1 + 1) \ln 2 - \ln (n_2 + 1) - \ln(8 / \eps)}
	=
	2^{-n_1} \cdot \frac{\eps}{8(n_2 + 1)}
	\enspace,
\end{align*}
where the last inequality follows from the definition of $m$. The lemma now follows from the last inequality by the union bound since $X$ was chosen as an arbitrary subset of $\cN_1$, and there are only $2^{n_1}$ such subsets.
\end{proof}

Using the above lemmata, we can now prove Theorem~\ref{thm:min_max_unconstrained_fail}.
\begin{proof}[Proof of Theorem~\ref{thm:min_max_unconstrained_fail}]
Let us denote by $X^*$ a set minimizing $\max_{Y \subseteq \cN_2} f(X^* \dotcup Y)$. By the definition of $g$ and the choice of $X'$ by Algorithm~\ref{alg:min_max_4_algorithm},
\[
	g(X')
	\leq
	g(X^*)
	=
	\tfrac{1}{m} \sum_{i = 1}^m f(X^* \dotcup Y_i)
	\leq
	\max_{Y \subseteq \cN_2} f(X^* \dotcup Y)
	=
	\min_{X \subseteq \cN_1} \max_{Y \subseteq \cN_2} f(X^* \dotcup Y)
	=
	\tau
	\enspace.
\]
Let us now denote by $\cE$ the event that $|g(X) - h(X)| \leq (\eps/20) \cdot h(X)$ for every set $X \subseteq \cN$. By Lemma~\ref{lem:surrogate_function}, $\cE$ happens with probability at least $1 - \eps / [8(n_2 + 1)]$. Furthermore, conditioned on $\cE$, we have
\begin{equation} \label{eq:set_bound}
	\max_{Y \subseteq \cN_2} f(X' \dotcup Y)
	\leq
	4 h(X')
	\leq
	\frac{4g(X')}{1 - \eps/20}
	\leq
	\frac{4\tau}{1 - \eps/20}
	\leq
	(4 + \eps/2)\tau
	\enspace,
\end{equation}
where the first inequality hold by Lemma~\ref{lem:g_approx_h}, and the last inequality holds for $\eps \in (0, 1/2)$. Since we always also have $\max_{Y \subseteq \cN_2} f(X' \dotcup Y) \geq \min_{X \subseteq \cN_1} \max_{Y \subseteq \cN_2} f(X \dotcup Y) = \tau$, the above inequality already proves that $\max_{Y \subseteq \cN_2} f(\hat{X} \dotcup Y)$ falls within the range $[\tau, (4 + \eps/2)\tau]$ whenever the event $\cE$ happens.

We now would like to show that the output value $(4 + \eps/2) \cdot g(X')$ of Algorithm~\ref{alg:min_max_4_algorithm} also falls within this range when the event $\cE$ happens. Since we already proved that $g(X') \leq \tau$, all we need to show is that $(4 + \eps/2) \cdot g(X')$ is at least $\tau$ condition on $\cE$. This is indeed the case since Inequality~\eqref{eq:set_bound} implies
\[
	(4 + \eps/2) \cdot g(X')
	\geq
	\frac{4}{1 - \eps/20} \cdot g(X')
	\geq
	\max_{Y \subseteq \cN_2} f(X' \dotcup Y)
	\geq
	\tau
	\enspace,
\]
where the first inequality holds for $\eps \in (0, 1/2)$. In conclusion, we have shown that when the event $\cE$ happens the value $(4 + \eps/2) \cdot g(X')$ returned by Algorithm~\ref{alg:min_max_4_algorithm} and the expression $\max_{Y \subseteq \cN_2} f(\hat{X} \dotcup Y)$ both fall within the range $(4 + \eps/2)$. Since the probability of the event $\cE$ is at least $1 - \eps / [8(n_2 + 1)]$, to complete the proof of the theorem it only remains to show that the expectation of $(4 + \eps/2) \cdot g(X')$ falls within the range $[\tau, (4 + \eps/2)\tau]$, which is what we do in the rest of this proof.

The inequality $\bE[(4 + \eps/2) \cdot g(X')] \leq (4 + \eps/2)\tau$ follows immediately from the above proof that we deterministically have $g(X') \leq \tau$. 
%We begin by plugging the inequality $g(X') \leq \tau$ and Inequality~\eqref{eq:set_bound} into the law of total expectation, which yields
%\[
	%\bE[g(X')]
	%\leq
	%\bE[g(X') \mid \cE] + \Pr[\neg \cE] \cdot \bE[g(X') \mid \neg \cE]
	%\leq
	%\left(1 + \frac{\eps}{20}\right)\tau + \frac{\eps}{8(n_2 + 1)} \cdot \tau
	%\leq
	%\left(1 + \frac{9\eps}{80}\right)\tau
%\]
%(in the last inequality we assumed $n_2 \geq 1$; if this is not the case, then the value of the expression $\min_{X \subseteq \cN_1} \max_{Y \subseteq \cN_2} f(X \dotcup Y)$ can calculated exactly using an algorithm for submodular minimization). Thus,
%\[
	%\bE[(4 + \eps/2) \cdot g(X')]
	%\leq
	%(4 + \eps/2) \cdot (1 + 9\eps / 80)\tau
	%\leq
	%(4 + \eps)\tau
	%\enspace,
%\]
%where the second inequality holds for any $\eps \in (0, 1/2)$.
%
Using Inequality~\eqref{eq:set_bound}, we can also get that, conditioned on $\cE$,
\begin{align*}
	g(X')
	\geq{} &
	\frac{(1 - \eps/20) \cdot \max_{Y \subseteq \cN_2} f(X' \dotcup Y)}{4}\\
	\geq{} &
	\frac{(1 - \eps/20) \cdot \min_{X \subseteq \cN_1} \max_{Y \subseteq \cN_2} f(X \dotcup Y)}{4}
	=
	\frac{(1 - \eps/20)\tau}{4}
	\enspace.
\end{align*}
Thus, we can use the law of total expectation to get
\[
	\bE[g(X')]
	\geq
	\Pr[\cE] \cdot \bE[g(X') \mid \cE]
	\geq
	\left(1 - \frac{\eps}{8(n_2 + 1)}\right) \cdot \frac{(1 - \eps/20)\tau}{4}
	\geq
	\left(\frac{1 - 9\eps / 80}{4}\right)\tau
	\enspace,
\]
which implies
\[
	\bE[(4 + \eps/2) \cdot g(X')]
	\geq
	\frac{(4 + \eps/2) \cdot (1 - 9\eps / 80)}{4} \cdot \tau
	\geq
	\tau
	\enspace,
\]
where the second inequality holds for $\eps \in [0, 1/2]$.
\end{proof}

\subsection{Proof of Theorem~\headerref{thm:min_max_sqrt}} \label{sec:min_max_sqrt}

In this section we prove Theorem~\ref{thm:min_max_sqrt}, which we repeat here for convenience.
\thmMinMaxSqrt*

The algorithm that we use to prove Theorem~\ref{thm:min_max_sqrt} is given as Algorithm~\ref{alg:square_root_approximation}. In this algorithm, and in its analysis, we use $n_1$ as a shorthand for $\cN_1$.
\begin{algorithm}
\DontPrintSemicolon
\caption{Iterative $X$ Growing} \label{alg:square_root_approximation}
Use an algorithm for submodular minimization to find a set $X_0 \in \arg \min_{X \subseteq \cN_1} f(X)$.\\
\For{$i = 1$ \KwTo $n_1 + 1$}
{
	Use the given oracle to find a set $Y_i \in \cF_2$ maximizing $f(X_{i - 1} \dotcup Y_i)$ up to a factor of $\alpha$ among all sets in $\cF_2$.\\
	Use an algorithm for submodular minimization to find a set $X'_i \in \arg \min_{X \subseteq \cN_1} [\sqrt{n_1} \cdot f(X \cup X_{i - 1}) + f(X \dotcup Y_i)]$.\label{line:prime_selection}\\
	\lIf{$X'_i \subseteq X_{i - 1}$}{\Return the set $X_{i - 1}$ and the value $\alpha \cdot f(X_{i - 1} \dotcup Y_i)$.}
	\lElse{Let $X_i \gets X_{i - 1} \cup X'_i$.}
}
\end{algorithm}

It will be useful to denote below by $X^*$ an arbitrary set in $\arg \min_{X \subseteq \cN_1} \max_{Y \in \cF_2} f(X \cup Y)$. Notice that the definitions of $X^*$ and $\tau$ imply together $\tau = \max_{Y \in \cF_2} f(X^* \cup Y)$. Thus, $f(X^* \cup Y) \leq \tau$ for every set $Y \in \cF_2$, and in particular, since $\varnothing \in \cF_2$ by assumption, $f(X^*) \leq \tau$. Below, we use $I$ to denote the number of iterations completed by the loop of Algorithm~\ref{alg:square_root_approximation}. Since the size of $X_i$ increases following every completed iteration, $I \leq n_1$. This implies that iteration $I + 1$ started, but stopped before completing since $X'_{I + 1}$ was a subset of $X_I$. Hence, $X_I$ is the output set of Algorithm~\ref{alg:square_root_approximation}. We begin the analysis of Algorithm~\ref{alg:square_root_approximation} with the following lemma.

\begin{lemma} \label{lem:X_increase_square}
For every integer $1 \leq i \leq I$,
\[
	f(X_i)
	=
	f(X_{i - 1} \cup X'_i)
	\leq
	f((X^* \cap X'_i) \cup X_{i - 1}) + \tau / \sqrt{n_1}
	\enspace.
\]
\end{lemma}
\begin{proof}
By the choice of $X'_i$, we have
\begin{align*}
	\sqrt{n_1} \cdot f(X'_i \cup X_{i - 1}) + f(X'_i \dotcup Y_i)
	\leq{} &
	\sqrt{n_1} \cdot f((X^* \cap X'_i) \cup X_{i - 1}) + f((X^* \cap X'_i) \dotcup Y_i)\\
	\leq{} &
	\sqrt{n_1} \cdot f((X^* \cap X'_i) \cup X_{i - 1}) + f(X^* \dotcup Y_i) + f(X'_i \dotcup Y_i)\\
	\leq{} &
	\sqrt{n_1} \cdot f((X^* \cap X'_i) \cup X_{i - 1}) + \tau + f(X'_i \dotcup Y_i)
	\enspace,
\end{align*}
where the second inequality follows from the submodularity and non-negativity of $f$, and the last inequality follows from the definition of $X^*$. The lemma now follows by rearranging the last inequality.
\end{proof}

\begin{corollary} \label{cor:X_value_square}
The output set $X_I$ of Algorithm~\ref{alg:square_root_approximation} obeys $f(X_I) \leq O(\sqrt{n_1}) \cdot \tau$.
\end{corollary}
\begin{proof}
If $I = 0$, then $X_I = X_0$, and by definition we have $f(X_I) \leq f(X^*) \leq \tau$. Therefore, we may assume from now on $I \geq 1$. Using Lemma~\ref{lem:X_increase_square}, we now get
\begin{align*}
	f(X_I) - f(X_0)
	\leq{} &
	\sum_{i = 1}^I [f(X_i) - f(X_{i - 1})]
	\leq
	\sum_{i = 1}^I [f(X^* \cap X'_i \mid X_{i - 1}) + \tau / \sqrt{n_1}]\\
	\leq{} &
	\sum_{i = 1}^I f(X^* \cap X'_i \mid X^* \cap X_{i - 1}) + \sqrt{n_1} \cdot \tau
	=
	f(X^* \cap X_I \mid X^* \cap X_0) + \sqrt{n_1} \cdot \tau\\
	\leq{} &
	f(X^* \cap X_I) - f(X_0) + \sqrt{n_1} \cdot \tau
	\enspace,
\end{align*}
where the penultimate inequality holds by the submodularity of $f$ and the observation that $I \leq n_1$, and the last inequality holds by the definition of $X_0$. Adding $f(X_0)$ to be both sides of the last inequality yields
\[
	f(X_I) - \sqrt{n_1} \cdot \tau
	\leq
	f(X^* \cap X_I)
	\leq
	f(X^*) + f(X_I) - f(X^* \cup X_I)
	\leq
	\tau + f(X_I) - f(X^* \cup X_I)
	\enspace,
\]
where the second inequality follows from the submodularity of $f$, and last inequality follows from the definition of $X^*$. To lower bound the term $f(X^* \cup X_I)$, we observe that since $I \geq 1$, the definition of $X'_I$ implies
\begin{align*}
	f(X^* \cup X_I)
	={} &
	f((X^* \cup X'_I) \cup X_{I - 1})
	\geq
	f(X'_I \cup X_{I - 1}) - \frac{f(X^* \mid X'_I \dotcup Y_I)}{\sqrt{n_1}}\\
	\geq{} &
	f(X_I) - \frac{f(X^*)}{\sqrt{n_1}}
	\geq
	f(X_I) - \tau / \sqrt{n_1}
	\enspace,
\end{align*}
where the second inequality uses the submodularity and non-negativity of $f$, and the last inequality holds by the definition of $X^*$. The corollary now follows by combining this inequality with the previous one.
\end{proof}

We are now ready to prove Theorem~\ref{thm:min_max_sqrt}.
\begin{proof}[Proof of Theorem~\ref{thm:min_max_sqrt}]
Since Algorithm~\ref{alg:square_root_approximation} outputted the set $X_I$, we must have $X'_{I + 1} \subseteq X_I$. Furthermore, by the choices of $Y_{I + 1}$ and $X'_{I + 1}$,
\begin{align*}
	\alpha^{-1} \cdot \max_{Y \in \cF_2} f(X_I \dotcup Y&)
	\leq
	f(X_I \dotcup Y_{I + 1})
	=
	f(X_I) + f(Y_{I + 1} \mid X_I)
	\leq
	f(X_I) + f(Y_{I + 1} \mid X'_{I + 1})\\
	\leq{} &
	f(X_I) + [\sqrt{n_1} \cdot f(X^* \cup X_I) + f(X^* \dotcup Y_{I + 1}) - \sqrt{n_1} \cdot f(X'_{I + 1} \cup X_I) - f(X'_{I + 1})]\\
	\leq{} &
	f(X_I) + \sqrt{n_1} \cdot f(X^* \mid X_I) + f(X^* \dotcup Y_{I + 1})
	\enspace,
\end{align*}
where the second inequality follows from the submodularity of $f$, and the last inequality holds by the non-negativity of $f$. Observe now that Corollary~\ref{cor:X_value_square} guarantees $f(X_I) \leq O(\sqrt{n}) \cdot \tau$, and the definition of $X^*$ guarantees $f(X^* \cup Y_{I + 1}) \leq \tau$. Furthermore, using the submodularity and non-negativity of $f$, we also get $f(X^* \mid X_I) \leq f(X^*) \leq \tau$. Plugging all these observations into the previous inequality yields
\[
	\alpha^{-1} \cdot \max_{Y \in \cF_2} f(X_I \dotcup Y)
	\leq
	f(X_I \dotcup Y_{I + 1})
	\leq
	O(\sqrt{n_1}) \cdot \tau + \sqrt{n_1} \cdot \tau + \tau
	=
	O(\sqrt{n_1}) \cdot \tau
	\enspace.
\]
Multiplying the last inequality by $\alpha$, we get the upper bound on $\max_{Y \in \cF_2} f(X_I \dotcup Y)$ and $\alpha \cdot f(X_I \dotcup Y_{I + 1})$ (the value outputted by Algorithm~\ref{alg:square_root_approximation}) promised in the theorem. The promised lower bound on these expressions also holds since the definition of $Y_I$ implies
\[
	\alpha \cdot f(X_I \dotcup Y_{I + 1})
	\geq
	\max_{Y \in \cF_2} f(X_I \dotcup Y)
	\geq
	\min_{X \subseteq \cN_1} \max_{Y \in \cF_2} f(X \dotcup Y)
	=
	\tau
	\enspace.
	\qedhere
\]
\end{proof}

\section{Experiments}\label{sec:exp}
In this section  we discuss five machine-learning applications: efficient prompt engineering, ride-share difficulty kernelization, adverserial attack on image summarization, robust ride-share optimization, and prompt engineering for dialog state tracking. Each one of these applications necessitates either max-min or min-max optimization on a jointly submodular function (with a cardinality constraint on the maximization part). To demonstrate the robustness of our suggested methods in this work, we empirically compare them against a few benchmarks.

In the max-min optimization applications, we compare the algorithm from Theorem~\ref{thm:min_submodular} (named below {\OurAlgorithm}) against $4$ benchmarks: (i) ``Random'' choosoing a random set of $k$ elements from $\cN_2$ as the set $Y$; (ii) ``Max-Only'' using a maximization algorithm to find the a set $Y$ that is (approximately) optimal against $X = \varnothing$; (iii) ``Top-$k$'' selecting a set $Y$ consisting of the top $k$ singletons $y \in \cN_2$, where each singleton is evaluated based on the corresponding worst case set $X$; and (iv) ``Best-Response'' simulating a best response dynamic between the minimization and maximization players, and outputing the set used by the maximization player after a given number of iterations. 
The Best-Response method is a widely used concept in game theory and optimization, first introduced in the seminal work by Vondrak et al.~\cite{von1947theory}. 
% The Best-Response method is a widely used concept in game theory and optimization. 
% {\color{red} It was first presented by \citet{von1947theory} in their seminal work ``Theory of Games and Economic Behavior", and has been widely studied since. Unfortunately, we were not able to prove theoretical guarantees for this method in our setting. Nevertheless, we provide an empirical evaluation of this method, studying both its empirical performance and convergence -- I am not sure this part is necessary. We can add this part to the conclusion as an open problem (amin)}.

In the min-max optimization applications, we study the algorithm from Theorem~\ref{thm:min_max_disjoint} (named below {\SingletonsAlg}) and a slightly modified version of the algorithm from Theorem~\ref{thm:min_max_sqrt} (named below {\IterativeAlg}). Out of the above $4$ benchmarks, the Random and Best-Response benchmarks still make sense in min-max settings with the natural adaptations. It was also natural to try to replace the Max-Only benchmark with a ``Min-Only'' benchmark, but such a benchmark would always output the empty set in our applications. Thus, we use instead a benchmark called ``Max-and-then-Min'' that returns a set $X$ that is optimal against the set $Y$ returned by Max-Only. See Appendix~\ref{app:implementation} for further detail about the various benchmarks, and the implementations of our algorithms. %used in the experiments.

%The empirical performance of Top-$k$ and Random in our experiments was clearly inferior compared to the other methods (by an order of roughly $2$). Therefore, we did not include them in the figures, and focused instead on comparing the performance of {\OurAlgorithm}, Best-Response and Max-Only.

%\section{Robust Prompt Engineering for Dialog State Tracking} 

% \pgfplotstableset{fixed zerofill,precision=2}
% \pgfmathfloatsetextprecision{3}
\subsection{Efficient prompt engineering} \label{ssc:prompt_engineering}
Consider the problem of designing efficient prompts for zero-shot in-context learning.
Following~\cite{si2023reliable}, we consider an open-domain question answering task: the goal is to answer questions from the SQuAD dataset~\cite{Rajpurkar2016squad} by prompting a large language model with $k$ relevant passages of text taken from a large corpus of Wikipedia articles. 
To get for each question an initial set of relevant candidate passages, $21$ million Wikipedia passages were embedded using a pretrained Contriever model~\cite{izacard2021contriever} and indexed using FAISS.\footnote{\url{https://github.com/facebookresearch/faiss}}
Then, for each question, the top $100$ passages were kept as candidates.

Large language models such as OpenAI's ChatGPT offer very impressive performance on natural language tasks via a public API. As the cost of making a prediction depends on the length of the input prompt, we propose to reduce the cost by \emph{jointly answering} similar questions with a common prompt, and thus, a single query to the GPT-3.5-turbo language model. To use this approach, we need to select a subset of passages that are effective on the set of \emph{answerable} questions, which we formulate as a combinatorial optimization problem. Specifically, let $\cN_1$ be a batch of questions and let $\cN_2$ be the union of all candidate passages.
(In general, $100 \leq |\cN_2| \leq 100 \cdot |\cN_1|$ since there may be significant overlap among candidates for questions on the same topic.)
Let $0 \leq s_{u,v} \leq 1$ be the cosine similarity between passage embedding $u$ and question embedding $v$.
Then, we define %the following score function.

% \begin{align}\label{eq:robust_qa_objective}
% f(X\dotcup Y)={}&\sum_{v\in \cN_1 \setminus X}\max_{u\in Y}s_{u,v} - \frac{\alpha}{|\cN_2|}\cdot\sum_{u\in Y}\sum_{v\in Y}s_{u,v}+ \beta\cdot\!\!\sum_{u\in \cN_1 \setminus X}\sum_{v\in Y}s_{u,v}  + \lambda\cdot|X|+|\cN_2|\enspace.
% \end{align}

\begin{align}\label{eq:robust_qa_objective}
f(X\dotcup Y)={}&\sum_{v\in \cN_1 \setminus X}\max_{u\in Y}s_{u,v} + \beta\cdot\!\!\sum_{u\in \cN_1 \setminus X}\sum_{v\in Y}s_{u,v}  + \lambda\cdot|X|\enspace.
\end{align}

Here $\lambda \geq 0$ and $\beta \geq 0$ are regularization parameters. The first term represents how well the passages of $Y$ cover the questions in $\cN_1 \setminus X$. 
% The second term encourages diversity among passages. 
For small values of $\beta$, the second term ensures $f$ increases in $|Y|$, and the last term controls the size of $X$.

\begin{restatable}{lemma}{lemRE}\label{lemma:RE}
The objective function~\eqref{eq:robust_qa_objective} is a non-negative jointly-submodular function.
\end{restatable}
% explain that \beta is small
\begin{proof}
Observe that the objective function~\eqref{eq:robust_qa_objective} is a conical combination of three terms. Below we explain why each one of these terms is non-negative and jointly-submodular, which immediately implies that the entire objective function also has these properties.

The second term is $\sum_{u\in \cN_1 \setminus X}\sum_{v\in Y}s_{u,v}$. This term can be viewed as the cut function of a directed bipartite graph in which the elements of $\cN_1$ and $\cN_2$ form the two sides of the graph, and for every $u \in \cN_1$ and $v \in \cN_2$ the graph includes an edge from $v$ to $u$ whose weight is $s_{u, v}$. Directed graph functions are known to be non-negative and submodular over the set of elements of the graph, which translates into joint-submodularity in our terminology since the both the elements of $\cN_1$ and $\cN_2$ are vertices of the graph.

The third term is $|X|$, which is a non-negative linear function, and thus, also jointly-submodular.

It remains to consider the first term, namely $\sum_{v\in \cN_1 \setminus X}\max_{u\in Y}s_{u,v}$. This term is clearly non-negative, so we concentrate below on proving that it is jointly submodular. Recall that $f$ is jointly-submodular if
 \[
 	f(u \mid Y' \dotcup X') \geq f(u \mid Y \dotcup X)
 	\qquad
 	\forall\; X' \subseteq X \subseteq \cN_1, Y' \subseteq Y \subseteq  \cN_2, u \in (\cN_1 \dotcup \cN_2) \setminus (X \dotcup Y)
 	\enspace.
 \]
 To prove that our objective function obeys this inequality, there are two scenarios to consider, based on whether $u$ belongs to the set $\cN_1$ or $\cN_2$.

 \paragraph{Case $\mathbf{1}$: The element $u$ belongs to $\cN_1$.} Here,
	\[
		f(u\mid Y'\dotcup X')-f(u \mid Y\dotcup X)=\max_{v\in Y}s_{u,v}-\max_{v\in Y'}s_{u,v}\geq 0
		\enspace.
	\]

 \paragraph{Case $\mathbf{2}$: the element $u$ belongs to $\cN_2$.} Observe that, in this case, 
 \begin{align*}
		f(u\mid Y'\dotcup X')
		={} &
		\sum_{\cN_1 \setminus X'} \max\{0, s_{u, v} - \max_{v' \in Y'} s_{u, v'}\}\\
		\geq{} &
		\sum_{\cN_1 \setminus X} \max\{0, s_{u, v} - \max_{v' \in Y} s_{u, v}\}
		=
		f(u\mid Y\dotcup X)
		\enspace.
		\qedhere
	\end{align*}
\end{proof}

By solving the max-min optimization $\max_{Y \subseteq \cN_2,|Y| \leq k}\min_{X\in\mathcal{N}_1} f(X\dotcup Y)$, we get the set $X$ of answerable questions, and a small set $Y$ of effective passages.
In our experiments we set $\beta=10^{-3}$, $\lambda=0.8$, and $k=10$,
and we group the SQuAD test set into batches of $25$ questions.
In addition to the heterogeneity introduced by crude batching, we removed $\delta=25\%$ of the candidates from $\cN_2$, leading to some questions having no relevant passages.

Table~\ref{tab:qa_results} shows the performance of the prompts for GPT-3.5-turbo obtained by {\OurAlgorithm} and various benchmarks. Each method is evaluated in terms of exact match accuracy and F1 score between predicted and ground truth answers. As a baseline, we also consider using the common prompt returned by the retrieval algorithm, but making a \emph{separate prediction} for each question in the cluster. We see our proposed joint prediction with a single prompt increases accuracy while on average requiring only 5.3\% of the tokens per question compared to separate prediction. Moreover, {\OurAlgorithm} has the highest accuracy among all retrieval algorithms used for joint prediction. Figure~\ref{fig:qa_example} shows a qualitative example of joint prediction for a batch of questions.

% Appendix~\ref{app:qa_examples} shows example prompts, questions, and answers obtained with different selection methods.
% explain separate predictions

\begin{table}[ht]
\caption{Open-domain question answering on SQuAD  using GPT-3.5-turbo. Best values are in \textbf{bold}.}\label{tab:qa_results}
\centering\small
\begin{tabular}{ccccc}
Prompting Method & Retrieval Algorithm & \ifdef{\noCiteInTable}{\hspace{-1mm}}{}Exact Match \% $\uparrow$ \ifdef{\noCiteInTable}{\hspace{-1mm}}{\hspace{-1mm}} & F1\% $\uparrow$ & Avg Tokens/Question $\downarrow$ {}\\ \toprule
\multirow{6}{*}{Joint Prediction} 
&           Random &                   18.6 &          29.7 &                 73.2 \\
&             None &                   22.5 &          33.9 &                 20.0 \\
&          Top-$k$ &                   25.0 &          37.0 &              \textbf{72.8} \\
&         Max-Only &                   25.9 &          37.5 &                 73.2 \\
&    Best-Response &                   25.6 &          37.0 &                 73.2 \\
& {\OurAlgorithm}  &                   \textbf{26.1} & \textbf{37.8} &        73.2 \\ \midrule
\multirow{6}{*}{Separate Prediction} 
&              Random &                    9.7 &          17.8 &               1356.3 \\
&                None &                   15.9 &          29.1 &                 40.1 \\
&             Top-$k$ &                   21.3 &          31.6 &               1338.8 \\
&            Max-Only &                   25.3 &          36.4 &               1348.7 \\
&       Best-Response &                   25.2 &          36.2 &               1348.7 \\
& {\OurAlgorithm}     &                   25.2 &          36.3 &               1348.7 \\ \bottomrule

% Selection Method &  Exact Match (Percent) &  F1 (Percent) &  Avg Tokens/Question \\
% \midrule
%         Max-Only &                   25.9 &          37.5 &                 73.2 \\
%     Robust Top-k &                   25.0 &          37.0 &                 72.8 \\
%    Best-Response &                   25.6 &          37.0 &                 73.2 \\
%    Min-As-Oracle &                   26.1 &          37.8 &                 73.2 \\
%           Random &                   18.6 &          29.7 &                 73.2 \\
%             None &                   22.5 &          33.9 &                 20.0 \\
%     Robust Top-k &                   21.3 &          31.6 &               1338.8 \\
%         Max-Only &                   25.3 &          36.4 &               1348.7 \\
%    Min-As-Oracle &                   25.2 &          36.3 &               1348.7 \\
%           Random &                    9.7 &          17.8 &               1356.3 \\
%    Best-Response &                   25.2 &          36.2 &               1348.7 \\

\end{tabular}
\end{table}

% Appendix~\ref{sec:app_dst_results} shows results for an additional prompt engineering application of dialog state tracking.

% effective/stable
% \section{Prompt engineering examples for question answering}\label{app:qa_examples}
% Figure~\ref{fig:qa_example} shows a qualitative example of joint prediction for a batch of questions.

% While joint prediction has an overall net benefit, in this example we see that some methods output answers in an inconsistent order.

\begin{figure}[th]
\includegraphics[width=\linewidth]{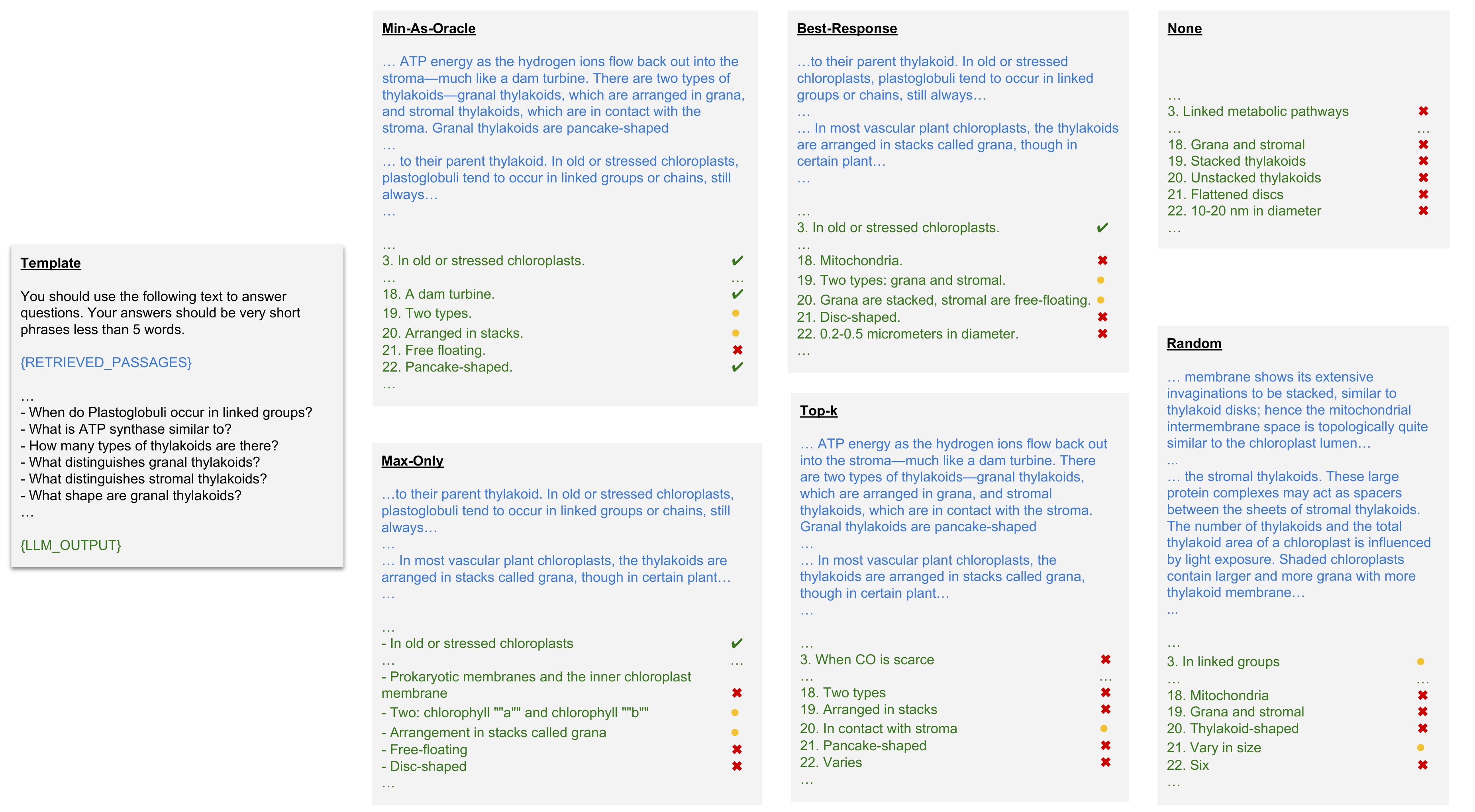}
\caption{Example of our proposed max-min formulation for jointly answering a batch of questions from the SQuAD dataset using GPT-3.5-turbo. Template prompt (left), followed by excerpts from the retrieved passages ({\color{NavyBlue}blue}) and generated answers ({\color{OliveGreen}green}). 
Exact match, partial match, and incorrect answer are denoted {\color{OliveGreen}\ding{52}}, {\color{Dandelion}\ding{108}}, and {\color{BrickRed}\ding{54}}, respectively. 
% Correct answers, partial matches, and incorrect answers are denoted {\color{green}\ding{52}}, {\color{yellow}\ding{108}}, and {\color{red}\ding{54}}, respectively. 
{\OurAlgorithm} retrieved two passages that are relevant to Questions 3, 18, 19, 20, 21, and 22, while the other selection algorithms retrieved only one or neither of them. Consequently, \textbf{Min-As-Oracle is best aligned with the ground truth answers, having the highest number of exact matches and the fewest number of hallucinations.}}
\label{fig:qa_example}
\end{figure}

% \section{Additional results for Section~\headerref{ssc:prompt_engineering}}\label{sec:app_dst_results}
% \section{Additional results for Robust Prompt Engineering}\label{sec:app_dst_results}

\subsection{Prompt engineering for dialog state tracking} \label{ssc:prompt_engineering_dst}
In this section, we consider the problem of selecting example (input, output) pairs for zero-shot in-context learning. 
In this application, the objective is to design prompts for the task of dialog state tracking (DST) on the MultiWOZ 2.4 data set~\cite{budzianowski2018multiwoz}. Following~\cite{hu2022incontext}, we recast this as a text-to-SQL problem in order to prompt the GPT-Neo~\cite{gpt-neo} and OpenAI Codex (\texttt{code-davinci-002})~\cite{chen2021codex} code generation models. These base models are adapted to DST with a combination of subset selection and in-context learning. First, a corpus of (previous dialog state, current dialog turn, SQL query) tuples is constructed from the training dialogs. Given a new input $u_0$, our prompt consists of 1) tabular representations of the dialog state ontology, 2) natural language instructions to query these tables using valid SQL given a task-oriented dialog turn, and 3) examples selected from the corpus by maximizing an objective function. 
% Here, the similarity score is computed by embedding both the input and the candidate with an SBERT model~\cite{reimers2019sbert} fine-tuned to reflect dialog state similarity, and then computing cosine similarity of the two embeddings.

% prompt engineering as robust optimzation
Let $u_0$ be the input query to a large language model. Each input $u_0$ contains a list of (key, value) pairs representing the previous dialog state predictions along with the text of the current dialog turn. We would like its prompt to be robust to incorrect predictions of the previous dialog states, as well as text variation such as misspellings. Let $\cN_1$ be a set of \textit{perturbed} inputs drawn from a small neighborhood around $u_0$. These perturbed inputs are constructed by randomly editing up to $2$ slots and/or values in the dialog state, and additionally dropping up to $15\%$ of tokens from the most recent dialog turn. In cases where $u_0$ is initially incorrect, examples that are similar to the perturbed inputs from $\cN_1$ improve the final prompt. Let $\cN = \cN_1 \cup \{u_0\}$, and let $\cN_2$ be the ground set of candidate examples. Given a set of examples $Y \subseteq \cN_2$ and a set of perturbed inputs $X \subseteq \cN_1$, we define the following score function.

\begin{align}\label{eq:robust_maxmin_objective}
f(X\dotcup Y)={}&\sum_{u\in \cN \setminus X}\sum_{v\in Y}s_{u,v}-\frac{\alpha}{|\cN_2|}\cdot\sum_{v\in Y}\sum_{u\in Y}s_{u,v}+\lambda\cdot|X|+|\cN_2|\enspace.
\end{align}
%
% \begin{align}\label{eq:robust_maxmin}
% \max_{T\in\mathcal{F}_2}\min_{S\in\mathcal{N}_1}\sum_{u\in \cN \setminus S}\sum_{v\in T}s_{u,v}-\frac{\alpha}{|\cN_2|}\cdot\sum_{v\in T}\sum_{u\in T}s_{u,v}+\lambda\cdot|S|+|\cN_2|.
% \end{align}
Here, $0\leq s_{u,v}\leq 1$ is the symmetric similarity score between examples $u,v$ (the similarity score is computed by embedding both examples with a pretrained SBERT model~\cite{reimers2019sbert}, and then computing cosine similarity of the two embeddings), and $\lambda \geq 0$ and $0 \leq \alpha \leq 1$ are regularization parameters. The parameter $\alpha$ explicitly trades off recommendation quality and diversity. Since we are interested in finding a set of candidate examples $Y$ that is good against the worst case set $X$ of perturbed inputs, we would like to optimize
$
\max_{Y \subseteq \cN_2,|Y| \leq k}\min_{X\in\mathcal{N}_1} f(X\dotcup Y)%\enspace,
$,
where $k$ is an upper bound on the number of examples to include in the prompt. 
By an argument similar to the proof of Lemma~\ref{lemma:1} (below), the objective function~\eqref{eq:robust_maxmin_objective} is a non-negative jointly-submodular function.
% The following lemma is proved in Section~\ref{ssc:lemRE_proof}.
% \begin{restatable}{lemma}{lemRE_dst}\label{lemma:RE_dst}
% The objective function~\eqref{eq:robust_maxmin_objective} is a non-negative jointly-submodular function.
% \end{restatable}

Initially, the GPT-Neo-Small generative model was evaluated with all possible combination of values for the regularization parameters from the grid $\lambda\in\{0, 0.5, 0.75 ,0.9 ,2.5\}$ and $\alpha\in\{0, 0.1,\allowbreak 0.3, 0.5, 0.7\}$. The best parameters  $(\lambda=0.9, \alpha=0.5)$ were then used for the other generative models. Following Section~5 of~\cite{hu2022incontext}, all retrieval models were evaluated on inputs obtained by randomly sampling $10\%$ of the MultiWOZ validation set, and all results were averaged over $3$ different candidate sets, which are randomly sampled $5\%$ subsets of the MultiWOZ training set. We set ($k=5$, $|\cN_1|=20$) for GPT-Neo models and ($k=10$, $|\cN_1|=4$) for the OpenAI Codex model.

% In our experiments, we have compared our algorithm from Theorem~\ref{thm:min_submodular} (named below \OurAlgorithm) using the standard greedy algorithm as ALG against the following baselines.
% \begin{itemize}
% \item Random: Select $k$ examples uniformly at random from the set of available candidates.
% \item Non-Robust Top-$k$: Use the implementation from~\cite{hu2022incontext} to retrieve the top $k$ candidates with a precomputed KD Tree.
% % \item Robust Top-$k$: Select the top $k$ examples $t_i$ which maximize $\min_{S \in N_1} f(t_i, S)$.
% \item Robust Top-$k$: Select the top $k$ singletons $x \in \cN_2$ which maximize \eqref{eq:robust_maxmin_objective} for the worst case $X$.
% \item Max-Only: (Approximately) maximize the objective function~\eqref{eq:robust_maxmin_objective} for $X=\varnothing$ using the standard greedy algorithm.
% \item Best-Response: Initialize $X_0\gets\varnothing$, and repeat for $5$ iterations of $Y_i \leftarrow \max_{Y \subseteq \cN_2, |Y| \leq k} f(X_{i-1} \dotcup Y) $ and $X_i \leftarrow \argmin_{X \subseteq \cN_1} f(X \dotcup Y_i)$,
% %
% where the maximization operation is done using the standard greedy algorithm, and therefore, is only approximate.
% \end{itemize}
In our experiments, we have compared {\OurAlgorithm} with the our standard max-min benchmarks Top-$k$, Max-Only and Best-Response, and also with a baseline termed ``Non-robust Top-$k$'' from~\cite{hu2022incontext}. For Top-$k$, Max-Only, Best-Response and {\OurAlgorithm}, we first retrieved a ground set of size $|\cN_2| = 100k$ candidates using the precomputed KD Tree, and only then selected the output set $Y$ using the retrieval algorithm.
Table~\ref{tab:dst_results} shows the Joint F1 score for each of the above-mentioned methods. Results for GPT-Neo models are averaged over $4$ random seeds. One can observe that prompting with our robust formulation outperforms the Non-Robust Top-$k$ baseline by as much as $1.5\%$. Among the algorithms using the robust formulation, our proposed algorithm {\OurAlgorithm} is consistently the best or 2nd best. Table~\ref{tab:dst_results} also shows that {\OurAlgorithm} achieves the highest objective value in all cases. Note that {\OurAlgorithm} has theoretical guarantees for both its convergence and approximation ratio, whereas Sections~\ref{app:ride_share} and~\ref{app:ride_share2} demonstrate that the Best-Response heuristic diverges for some instances.
% \ee{mention {\OurAlgorithm} always has the highest objective value here}

% \newcommand{\noCiteInTable}{}
\begin{table}[ht]
\caption{Dialog state tracking performance and objective values for different language models and retrieval algorithms. Best values are in \textbf{bold}.}\label{tab:dst_results}
\centering\small
% \begin{tabular}{c|c|c|c|c}
\begin{tabular}{cccc}
% % tex
% remove JGA and keep Joint F1 only
Generative Model & Retrieval Algorithm & \ifdef{\noCiteInTable}{\hspace{-1mm}}{}Joint F1\ifdef{\noCiteInTable}{} & Objective Value\\ \toprule
\multirow{6}{*}{GPT-Neo-Small} & Random & 0.0480 & 25.259 \\
& Non-robust Top-$k$\ifdef{\noCiteInTable}{}{~\cite{hu2022incontext}} & 0.3249  & 26.165\\
& Top-$k$ & 0.2787  & 26.125\\
& Max-Only & 0.2783 & 26.134\\
& Best-Response & \textbf{0.3251} & 26.165 \\
& {\OurAlgorithm} & 0.3022  & \textbf{26.168} \\ \midrule
\multirow{6}{*}{GPT-Neo-Large} & Random & 0.2275 & 25.254 \\
& Non-robust Top-$k$\ifdef{\noCiteInTable}{}{~\cite{hu2022incontext}}  & 0.4872 & 26.164\\
& Top-$k$ & 0.4821 & 26.127\\
& Max-Only & 0.4830 & 26.134\\
& Best-Response & 0.4845 & 26.165 \\ 
& {\OurAlgorithm} & \textbf{0.5020} & \textbf{26.168} \\ \midrule
\multirow{6}{*}{Codex-Davinci} & Random & 0.8273 & 17.410 \\
& Non-robust Top-$k$\ifdef{\noCiteInTable}{}{~\cite{hu2022incontext}} & 0.8929 & 19.021 \\
& Top-$k$ & 0.8974 & 18.954\\
& Max-Only  & 0.8913 & 18.953\\
& Best-Response & 0.8972 & 19.022 \\
& {\OurAlgorithm} & \textbf{0.8991} & \textbf{19.027} \\ \bottomrule

\end{tabular}
% {results/test.csv}
\end{table}

% In Table~\ref{tab:dst_results_objective}, we provide the objective value results corresponding to the experiments reported in Table~\ref{tab:dst_results}.

% \begin{table}[ht]
% \caption{Objective values for dialog state tracking for different language models and retrieval algorithms. Best values are in \textbf{bold}.}\label{tab:dst_results_objective}
% \centering\small
% % \begin{tabular}{c|c|c|c|c}
% \begin{tabular}{ccc}
% \input{results/results_prompting_objective.tex}
% \end{tabular}
% % {results/test.csv}
% \end{table}

\subsection{Ride-share difficulty kernelization}\label{app:ride_share}
Consider a regulator overseeing the taxi companies licensed to operate within a given city. The regulator wants to make sure that the taxi companies give a fair level of service to all parts of the city, rather than concentrating on the most profitable neighborhoods. However, checking that this is indeed the case is not trivial since often the limited number of taxis available implies that some locations must remain poorly served. Our objective in this section is to give the regulator a small set (kernel) of locations that that capture the difficulty of the problem faced by the taxi company in the sense that the locations in the set cannot be served well (on average) regardless of how the taxi companies choose the waiting locations for their taxis.

Formally, given a set $\cN_1$ of (client) pickup locations, and a set $\mathcal{N}_2$ of potential waiting locations for taxies, we define the following score function to capture the convenience of serving all the locations of $\cN_1 \setminus X$ by locating taxis at locations $Y$.\footnote{It would have been more natural to define $X$ as the set of locations to service. However, this would have resulted in an objective function that is only disjointly submodular.}
{\allowdisplaybreaks\begin{align}\label{func:img}
	f(X \dotcup Y)
	={} 
\sum_{v\in \mathcal{N}\setminus X}  \max_{u\in Y}s_{u,v}-\frac{1}{|\mathcal{N}_2|}\sum_{u\in Y}\sum_{v\in Y} s_{u,v} + \lambda\cdot|X|
\enspace.
\end{align}}%
Here, $s_{u,v}$ is a ``convenience score" which, given a customer location $u = (x_u, y_u)$ and a waiting driver location $v = (x_v, y_v)$,\footnote{Each location is specified by a (latitude, longitude) coordinate pair.} represents the ease of accessing $u$ from $v$. Following \cite{mitrovic2018data}, we set $s_{u, v} \triangleq 2 - \frac{2}{1 + e^{-200d(u,v)}}$, where $d(u,v) = |x_u - x_v| + |y_u - y_v|$ is the Manhattan distance between the two points. The value $\lambda \in [0, 1]$ is a regularization parameter whose use is discussed below.
% Some properties of this objective function are given by the next lemma, whose proof can be found in Appendix~\ref{app:omitted}.
Some properties of this objective function are given by the next lemma.

\begin{restatable}{lemma}{lemmaOne}\label{lemma:1}
The objective function~\eqref{func:img} is a non-negative jointly-submodular function.
\end{restatable}
\begin{proof}
First, we shall establish that the objective function is non-negative by demonstrating that the first term of the function is consistently greater than the subsequent term. This is established through the following inequality.
\[
    \sum_{v\in \mathcal{N} \setminus X}  \max_{u\in Y}s_{u,v}
    \geq
    \sum_{v\in \mathcal{N}_2} \max_{u\in Y}s_{u,v}
    \geq
    \frac{1}{|Y|} \cdot \sum_{v\in \mathcal{N}_2} \sum _{u\in Y}s_{u,v}
    \geq
    \frac{1}{|\mathcal{N}_2|} \cdot \sum_{v\in Y} \sum _{u\in Y}s_{u,v}\enspace.
\]

Next, we demonstrate that the objective function $f(X,Y)$ is jointly-submodular. Recall that $f$ is jointly-submodular if
\[
	f(u \mid Y' \dotcup X') \geq f(u \mid Y \dotcup X)
	\qquad
	\forall\; X' \subseteq X \subseteq \cN_1, Y' \subseteq Y \subseteq  \cN_2, u \in (\cN_1 \dotcup \cN_2) \setminus (X \dotcup Y)
	\enspace.
\]
To prove that our objective function obeys this inequality, there are two scenarios to consider, based on whether $u$ belongs to the set $\cN_1$ or $\cN_2$.

\paragraph{Case $\mathbf{1}$: The element $u$ belongs to $\cN_1$.}
Here, $f(u \mid Y'\dotcup X')-f(u\mid Y\dotcup X) = \max_{v\in Y}s_{u,v}-\max_{v\in Y'}s_{u,v}\geq 0$.

\paragraph{Case $\mathbf{2}$: The element $u$ belongs to $\cN_2$.}
%
%First, observe the following.
%
Let $\phi(Y,w,J)=\sum_{v\in J} (\max_{u\in Y + w}s_{u,v}-\max_{u\in Y}s_{u,v})$; and note that, for every two sets $Y'\subseteq Y \subseteq \cN_2$, set $J\subseteq \cN_1$ and element $u\in\cN_2\setminus T$, $\phi(Y', u,J)\geq\phi(Y,u,J)$. Using this notation, we get that in this case (the case of $u \in \cN_2$)
\begin{itemize}
    \item $f(u \mid Y'\dotcup X') = \phi(Y',\{u\},\cN\setminus X') - \frac{1}{|\cN_2|}\left(2 \cdot \sum_{v\in Y'}s_{u,v} + s_{u,u}\right)$, and
    \item $f(u \mid Y\dotcup X) = \phi(Y,\{u\},\cN\setminus X) - \frac{1}{|\cN_2|}\left(2 \cdot \sum_{v\in Y}s_{u,v} + s_{u,u}\right)$.
\end{itemize}
%In order to show that the function is submodular in this case, the following inequality should hold.$$f(\{u\} \mid T'\dotcup S')-f(\{u\}\mid T\dotcup S)\geq 0$$ 
Thus,
\[f(u \mid Y'\dotcup X')-f(u\mid Y\dotcup X) = \phi(Y',\{u\},\cN\setminus X') - \phi(Y,\{u\},\cN\setminus X) + \frac{2}{|\cN_2|}\cdot\sum_{v\in Y\setminus Y'}s_{u,v}\geq 0 \enspace,\]
where the last inequality holds since  $\phi(Y',\{u\},\cN\setminus X') - \phi(Y,\{u\},\cN\setminus X) \geq 0$ and $s_{u,v}\geq0$ for any $u,v\in\cN$.
\end{proof}

Recall that we are looking for a kernel set $\cN_1 \setminus X$ of pickup locations that cannot be served well (on average) by any choice of $k$ locations for taxis ($k$ is determined by the number of taxis available). To do that, we solve the max-min optimization problem given by $\min_{X \subseteq \cN_1} \max_{Y \subseteq \cN_2, |Y| \leq k} f(X \dotcup Y)$. The regularization parameter $\lambda$ can now be used to control the size the kernel set returned.

In our experiments for this application, we have used the Uber data set~\cite{uber}, which includes real-life Uber pickups in New York City during the month of April in the year 2014. To ensure computational tractability, in each execution of our experiments, we randomly selected from this data set a subset of $|\mathcal{N}_1|=\numprint{6,000}$ pickup locations within the region of Manhattan. Then, we randomly selected a subset of $400$ pickup locations from the set $\cN_1$ to constitute the set $\cN_2$ (we treat locations in $\cN_1$ and $\cN_2$ as distinct even if they are identical, to guarantee that $\cN_1$ and $\cN_2$ are disjoint as is technically required).

In the first experiment, we fixed the maximum number of waiting locations to be $8$, and varied $\lambda$. Figure~\ref{img:minlocation1} depicts the outputs for this experiment for {\SingletonsAlg}, {\IterativeAlg} (with $\beta = 0.5$) and three benchmarks (averaged over $10$ executions of the experiment). One can observe that both {\IterativeAlg} and {\SingletonsAlg} surpasses the performance of all benchmarks for almost all values of $\lambda$. We note that in both this experiment and the next one the standard error of the mean is less than $10$ for all data points.

In the second experiment, we fixed the regularization parameter $\lambda$ to $0.2$ and varied the number of allowed waiting locations. The results of this experiment are depicted by Figure~\ref{img:minlocation2} (again, averaged over $10$ executions of the experiment). Once again, {\IterativeAlg} and {\SingletonsAlg} demonstrate superior performance compared to the benchmarks for almost all values of $k$.

As the third experiment for this application, we conducted a more in depth analysis of the Best-Response technique. Figure~\ref{img:minlocation3} graphically presents the objective function value obtained by a typical execution of Best-Response after a varying number of iterations (for $\lambda = 0.5$ and an upper bound of $20$ on the number of waiting locations). It is apparent that Best-Response does not converge for this execution. Furthermore, both our suggested algorithms demonstrate better performance even with respect to the best performance of Best-Response for any number of iterations between $1$ and $50$.

\begin{figure}[htbp]
\begin{subfigure}[t]{0.31\textwidth}
  \begin{tikzpicture}[scale=0.5] 
\begin{axis}[
    xlabel = {$\lambda$ value},
    ylabel = {Function Value},
    xmin=0.1, xmax=1,
    ymin=1000, ymax=3000,
		legend cell align=left,
		legend style={at={(1, 1.07)}},
		%error bars/y dir=both,
    %error bars/y explicit
		]
		\pgfplotstableread{CSV_Location_Minmax/Location_k8_random_mean.csv}\Random
		%\pgfplotstablecreatecol[copy column from table={CSV_Location_Minmax/Location_k8_greed_STD.csv}{[index] 0}] {error} {\MaxOnly}
		\addplot [name path = random, darkgray, mark = o, mark size=3pt] table [x expr=0.1+\coordindex/10, y index=0, y error index = 1] {\Random};
		\addlegendentry{Random}
		\pgfplotstableread{CSV_Location_Minmax/Location_k8_greed_mean.csv}\MaxOnly
		\pgfplotstablecreatecol[copy column from table={CSV_Location_Minmax/Location_k8_greed_STD.csv}{[index] 0}] {error} {\MaxOnly}
		\addplot [name path = theirs, black, mark = triangle*, mark size=3pt] table [x expr=0.1+\coordindex/10, y index=0, y error index = 1] {\MaxOnly};
		\addlegendentry{Max-and-then-Min}
		\pgfplotstableread{CSV_Location_Minmax/Location_k8_BR_mean.csv}\BestResponse
		\pgfplotstablecreatecol[copy column from table={CSV_Location_Minmax/Location_k8_br_STD.csv}{[index] 0}] {error} {\BestResponse}
		\addplot [name path = theirs, red, mark = square*] table [x expr=0.1+\coordindex/10, y index=0, y error index = 1] {\BestResponse};
		\addlegendentry{Best-Response}
  \pgfplotstableread{CSV_Location_Minmax/Location_k8_sing_mean.csv}\ourss
		\pgfplotstablecreatecol[copy column from table={CSV_Location_Minmax/Location_k8_sing_STD.csv}{[index] 0}] {error} {\ourss}
		\addplot [name path = our, black!30!green, mark = diamond*, mark size=3pt] table [x expr=0.1+\coordindex/10, y index=0, y error index=1] {\ourss};
		\addlegendentry{\SingletonsAlg}
		\pgfplotstableread{CSV_Location_Minmax/Location_k8_reg_alpha5_mean.csv}\ours
		\pgfplotstablecreatecol[copy column from table={CSV_Location_Minmax/Location_k8_reg_STD.csv}{[index] 0}] {error} {\ours}
		\addplot [name path = our, blue, mark = *] table [x expr=0.1+\coordindex/10, y index=0, y error index=1] {\ours};
		\addlegendentry{\IterativeAlg}
  %   		\addplot [name path = theirs, orange, mark = triangle*, mark size=3pt, mark repeat=2] table [x expr=\coordindex, y index=0] {CSV/Location_topk_gamma35.csv};
		% \addlegendentry{Top k}
	\end{axis}\end{tikzpicture}\caption{Results for $8$ waiting locations}
    \label{img:minlocation1}
 \end{subfigure}\hfill
\begin{subfigure}[t]{0.31\textwidth}
  \begin{tikzpicture}[scale=0.5] 
\begin{axis}[
    xlabel = {Number of waiting locations},
    ylabel = {Function Value},
    xmin=4, xmax=11,
    ymin=1000, ymax=3000,
		legend cell align=left,
		legend style={at={(0.55, 1.07)}},
		%error bars/y dir=both,
    %error bars/y explicit
		]
		\pgfplotstableread{CSV_Location_Minmax/Location_gamma5_k12_random_mean.csv}\Random
		%\pgfplotstablecreatecol[copy column from table={CSV_Location_Minmax/Location_gamma2_greed_STD.csv}{[index] 0}] {error} {\MaxOnly}
		\addplot [name path = random, darkgray, mark = o, mark size=3pt] table [x expr=\coordindex+4, y index=0, y error index = 1] {\Random};
		\addlegendentry{Random}
        \pgfplotstableread{CSV_Location_Minmax/Location_gamma2_greed_mean.csv}\MaxOnly
		\pgfplotstablecreatecol[copy column from table={CSV_Location_Minmax/Location_gamma2_greed_STD.csv}{[index] 0}] {error} {\MaxOnly}
		\addplot [name path = theirs, black, mark = triangle*, mark size=3pt] table [x expr=\coordindex+4, y index=0, y error index = 1] {\MaxOnly};
		\addlegendentry{Max-and-then-Min}
		\pgfplotstableread{CSV_Location_Minmax/Location_gamma2_BR_mean.csv}\BestResponse
		\pgfplotstablecreatecol[copy column from table={CSV_Location_Minmax/Location_gamma2_br_STD.csv}{[index] 0}] {error} {\BestResponse}
		\addplot [name path = theirs, red, mark = square*] table [x expr=\coordindex+4, y index=0, y error index = 1] {\BestResponse};
		\addlegendentry{Best-Response}
		\pgfplotstableread{CSV_Location_Minmax/Location_gamma2_sing_mean.csv}\ourss
		\pgfplotstablecreatecol[copy column from table={CSV_Location_Minmax/Location_gamma2_sing_STD.csv}{[index] 0}] {error} {\ourss}
		\addplot [name path = our, black!30!green, mark = diamond*, mark size=3pt] table [x expr=\coordindex+4, y index=0, y error index=1] {\ourss};
		\addlegendentry{\SingletonsAlg}
  	\pgfplotstableread{CSV_Location_Minmax/Location_gamma2_reg_alpha5_mean.csv}\ours
		\pgfplotstablecreatecol[copy column from table={CSV_Location_Minmax/Location_gamma2_reg_STD.csv}{[index] 0}] {error} {\ours}
		\addplot [name path = our, blue, mark = *] table [x expr=\coordindex+4, y index=0, y error index=1] {\ours};
		\addlegendentry{\IterativeAlg}
	\end{axis}\end{tikzpicture}\caption{Results for $\lambda = 0.2$}
    \label{img:minlocation2}
 \end{subfigure}\hfill
 \begin{subfigure}[t]{0.31\textwidth}
  \begin{tikzpicture}[scale=0.5] 
\begin{axis}[
    xlabel = {Iterations},
    ylabel = {Function Value},
    xmin=0, xmax=50,
    ymin=1910, ymax=2200,
		%legend cell align=left,
		%legend style={at={(0.9,0.3)}}
		]
		\addplot [name path = our, red, mark = *, mark repeat=1] table [x expr=\coordindex, y index=0] {CSV_Location_Minmax/BR_mean2.csv};
		%\addlegendentry{Best-Response}
	\end{axis}\end{tikzpicture}\caption{Behavior of Best-Response for $\lambda = 0.2$ and $8$ waiting locations}
    \label{img:minlocation3}
 \end{subfigure}

\caption{Empirical results for ride-share difficulty kernelization. Figures~(\subref{img:minlocation1}) and~(\subref{img:minlocation2}) compare the performance of our algorithms {\SingletonsAlg} and {\IterativeAlg} with $3$ benchmarks for different value of $\lambda$ and bounds on the number of weighting locations. Figure~(\subref{img:minlocation3}) depicts the value of the output of the Best-Response method as a function of the number of iterations performed.}
\label{fig:kernelization_graphs}
\end{figure}
\subsection{Robust ride-share optimization}\label{app:ride_share2}
In the ``Robust Ride-Share Optimization'' application, our primary objective is to determine the most suitable waiting locations for idle taxi drivers based on taxi order history. This problem was previously formalized as a traditional facility location problem~\cite{mitrovic2018data}. However, in the current work, we look for a more robust set of waiting locations. Often some locations are inaccessible (for example, due to road maintenance). Hence, we wish to find a robust set of waiting locations that effectively minimizes the distance between each customer and her closest driver even when some of the locations are inaccessible. 

The objective function we use to solve the above problem is technically identical to the jointly-submodular function given by~\eqref{func:img}. However, now $\cN_1$ represents the (client) pickup locations that might be inaccessible due to traffic (while $\mathcal{N}_2$ remains the set of potential waiting locations for idle drivers). Furthermore, we now need to perform max-min optimization on this objective function since we look for a set $Y$ of up to $k$ waiting locations that is good regardless of which pickup locations become inaccessible.

In our experiments, we used again the Uber data set~\cite{uber} (see Section~\ref{app:ride_share}). To ensure computational tractability, in each execution of our experiments, we randomly selected from this data set a subset of $|\mathcal{N}|=\numprint{6,000}$ pickup locations within the region of Manhattan. Then, we chose the set $\mathcal{N}_1$ to consist of all the pickup locations that have a latitude value greater than $40.8$, or less than $40.73$. This set represents the pickup locations that are potentially unavailable (for example, due to traffic). 
Furthermore, we randomly selected a subset of $400$ pickup locations from the set $\cN$ to constitute the set $\cN_2$. This set represents the potential waiting locations for idle drivers.% The size $400$ of this subset was chosen in order to ensure computational tractability while still providing a sufficient number of options for the optimization problem.

In the first experiment, we fixed the regularization parameter $\lambda$ to $0.35$ and varied the number of allowed waiting locations. Figure~\ref{img:location1} depicts the outputs for this experiment for our algorithm {\OurAlgorithm} and two benchmarks (averaged over $10$ executions of the experiment). One can observe that {\OurAlgorithm} consistently surpasses the two benchmarks. The two other benchmarks (Random and Top-$k$) where also included in this experiment and the next one, but are excluded from the figures since their outputs are worse by a factor of at least $2$ comapred to the presented methods. We also note that in both experiments the standard error of the mean is less than $10$ for all data points.

In the second experiment, we fixed the maximum number of waiting locations to be $15$, and varied $\lambda$. The results of this experimented are depicted by Figure~\ref{img:location2} (again, averaged over $10$ executions of the experiment). Once again, our proposed method, {\OurAlgorithm}, demonstrates superior performance compared to the bechnmarks, with the gap being significant for lower values of $\lambda$. %The performance of Best-Response is again comparable to that of Max-Only, except for very small values of $\lambda$.

As the third experiment for this application, we conducted a more in depth analysis of the Best-Response technique. Figure~\ref{img:location3} graphically presents the objective function value obtained by a typical execution of Best-Response after a varying number of iterations (for $\lambda = 0.35$ and an upper bound of $20$ on the number of waiting locations). It is apparent that Best-Response does not converge for this execution. Furthermore, {\OurAlgorithm} demonstrates better performance even with respect to the best performance of Best-Response for any number of iterations between $1$ and $50$.

\begin{figure}[htbp]
\begin{subfigure}[t]{0.31\textwidth}
  \begin{tikzpicture}[scale=0.5] 
\begin{axis}[
    xlabel = {Number of waiting locations},
    ylabel = {Function Value},
    xmin=10, xmax=20,
    ymin=1720, ymax=2400,
		legend cell align=left,
		legend style={at={(0.9,0.3)}},
		%error bars/y dir=both,
    %error bars/y explicit
		]
		\pgfplotstableread{CSV_Location/Location_Greedy_gamma35_10_mean.csv}\MaxOnly
		\pgfplotstablecreatecol[copy column from table={CSV_Location/Location_Greedy_gamma35_10_std.csv}{[index] 0}] {error} {\MaxOnly}
		\addplot [name path = theirs, black, mark = triangle*, mark size=3pt] table [x expr=10+\coordindex, y index=0, y error index = 1] {\MaxOnly};
		\addlegendentry{Max-Only}
		\pgfplotstableread{CSV_Location/Location_bestresponse_gamma35_10_mean.csv}\BestResponse
		\pgfplotstablecreatecol[copy column from table={CSV_Location/Location_bestresponse_gamma35_10_std.csv}{[index] 0}] {error} {\BestResponse}
		\addplot [name path = theirs, red, mark = square*] table [x expr=10+\coordindex, y index=0, y error index = 1] {\BestResponse};
		\addlegendentry{Best-Response}
		\pgfplotstableread{CSV_Location/Location_maxmin_gamma35_10_mean.csv}\ours
		\pgfplotstablecreatecol[copy column from table={CSV_Location/Location_maxmin_gamma35_10_std.csv}{[index] 0}] {error} {\ours}
		\addplot [name path = our, blue, mark = *] table [x expr=10+\coordindex, y index=0, y error index=1] {\ours};
		\addlegendentry{\OurAlgorithm}
  %   		\addplot [name path = theirs, orange, mark = triangle*, mark size=3pt, mark repeat=2] table [x expr=\coordindex, y index=0] {CSV/Location_topk_gamma35.csv};
		% \addlegendentry{Top k}
	\end{axis}\end{tikzpicture}\caption{$\lambda=0.35$}
    \label{img:location1}
 \end{subfigure}\hfill
\begin{subfigure}[t]{0.31\textwidth}
  \begin{tikzpicture}[scale=0.5] 
\begin{axis}[
    xlabel = {$\lambda$ values},
    ylabel = {Function Value},
    xmin=0, xmax=1,
    ymin=1000, ymax=2500,
		legend cell align=left,
		legend style={at={(0.9,0.3)}},
		%error bars/y dir=both,
    %error bars/y explicit
		]
		\pgfplotstableread{CSV_Location/Location_Greedy_k15_10_mean.csv}\MaxOnly
		\pgfplotstablecreatecol[copy column from table={CSV_Location/Location_Greedy_k15_10_std.csv}{[index] 0}] {error} {\MaxOnly}
		\addplot [name path = theirs, black, mark = triangle*, mark size=3pt] table [x expr=\coordindex/10, y index=0, y error index = 1] {\MaxOnly};
		\addlegendentry{Max-Only}
		\pgfplotstableread{CSV_Location/Location_bestresponse_k15_10_mean.csv}\BestResponse
		\pgfplotstablecreatecol[copy column from table={CSV_Location/Location_bestresponse_k15_10_std.csv}{[index] 0}] {error} {\BestResponse}
		\addplot [name path = theirs, red, mark = square*] table [x expr=\coordindex/10, y index=0, y error index = 1] {\BestResponse};
		\addlegendentry{Best-Response}
		\pgfplotstableread{CSV_Location/Location_maxmin_k15_10_mean.csv}\ours
		\pgfplotstablecreatecol[copy column from table={CSV_Location/Location_maxmin_k15_10_std.csv}{[index] 0}] {error} {\ours}
		\addplot [name path = our, blue, mark = *] table [x expr=\coordindex/10, y index=0, y error index=1] {\ours};
		\addlegendentry{\OurAlgorithm}
  %   		\addplot [name path = theirs, orange, mark = triangle*, mark size=3pt, mark repeat=2] table [x expr=\coordindex, y index=0] {CSV/Location_topk_k10.csv};
		% \addlegendentry{Top k}
	\end{axis}\end{tikzpicture}\caption{$15$ waiting locations}
    \label{img:location2}
 \end{subfigure}\hfill
 \begin{subfigure}[t]{0.31\textwidth}
  \begin{tikzpicture}[scale=0.5] 
\begin{axis}[
    xlabel = {Iterations},
    ylabel = {Function Value},
    xmin=0, xmax=50,
    ymin=2325, ymax=2380,
		legend cell align=left,
		legend style={at={(0.9,0.3)}}]
		\addplot [name path = our, red, mark = *, mark repeat=1] table [x expr=\coordindex, y index=0] {CSV/Location_bestresponse_gamma35_k=20_100iterations.csv};
		%\addlegendentry{Best-Response}
	\end{axis}\end{tikzpicture}\caption{Behavior of Best-Response for $\lambda = 0.35$ and $20$ waiting locations}
    \label{img:location3}
 \end{subfigure}

\caption{Empirical results for robust ride-share optimization. Figures~(\subref{img:location1}) and~(\subref{img:location2}) compare the performance of our algorithm {\OurAlgorithm} with $2$ benchmarks for different value of $\lambda$ and bounds on the number of weighting locations. Figure~(\subref{img:location3}) depicts the value of the output of the Best-Response method as a function of the number of iterations performed.}
\label{fig:Location_graphs}
\end{figure}

Our last experiment for this section aims to give a more intuitive point of view on the performance of our algorithm ({\OurAlgorithm}). Figure~\ref{fig:maps} depicts the results of this algorithm on maps of Manhattan for three different values of $\lambda$ ($0.2$, $0.4$ and $0.8$). To make the maps easy to read, we allowed the algorithm to select only $6$ waiting locations for idle drivers, and the locations suggested by the algorithm are marked with red triangles on the maps. We have also marked on the maps the pick up locations of $\cN$. The black dots represent the waiting locations that are inaccessible, while the light gray dots indicate the accessible pickup locations. Intuitively, the regularization parameter $\lambda$ captures in this application the probability of pickup locations in $\cN_1$ to be accessible. For example, when $\lambda=0$, it is assumed that all locations in $\cN_1$ are inaccessible, whereas $\lambda=1$ means that all locations in $\cN_1$ are assumed to be accessible. This intuitive role of $\lambda$ is demonstrated in Figure~\ref{fig:maps} in the following sense. As the value of $\lambda$ increases, the number of red triangles in the figure inside the areas of the black dots tends to increase, and furthermore, the locations of theses triangles are pushed deeper into these areas.

\begin{figure}[htbp]
\begin{subfigure}[t]{0.3\textwidth}
  \includegraphics[width=\linewidth]{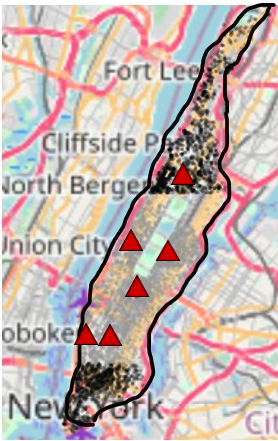}
\caption{$\lambda=0.2$}
\label{fig:map1}
\end{subfigure}\hfill
\begin{subfigure}[t]{0.3\textwidth}
    \includegraphics[width=\linewidth]{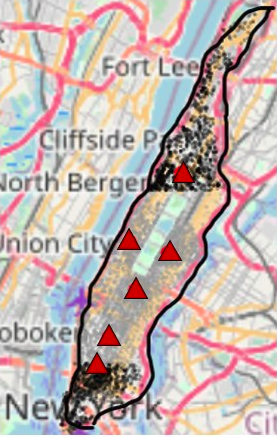}
\caption{$\lambda=0.4$}
\label{fig:map2}
\end{subfigure}\hfill
\begin{subfigure}[t]{0.3\textwidth}
    \includegraphics[width=\linewidth]{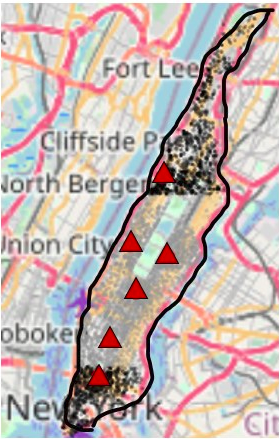}
\caption{$\lambda=0.8$}
\label{fig:map3}
\end{subfigure}
\caption{The results of {\OurAlgorithm} for $3$ different values of $\lambda$. The red triangles represent waiting locations chosen by the algorithm, the light gray dots represent always accessible pick-up locations, and the black dots represent possibly inaccessible pick-up locations.}
\label{fig:maps}
\end{figure}

\subsection{Adversarial attack on image summarization} \label{ssc:image_summarization}
In this section we consider the application of ``Adversarial Attack on Image Summarization", which is an attack version of an application studied by many previous works (see, e.g.,~\cite{mitrovic2018data,mualem2022using,tschiatschek2014learning}). The setting for this application includes a collection of images from $\ell$ disjoint categories (such as birds, airplanes or cats), and a user that specifies $r \in [\ell]$ categories of interest. In the classical version of this application, the objective is to construct a subset of $k$ images summarizing the images belonging to the categories specified by the user. However, here we are interested in mounting an attack against this summarization task. Specifically, our goal is to add a few additional images to the original set of images in a way that undermines the quality of any subsequently chosen summarizing subset.

Formally, we have in this application a (completed) similarity matrix $M$ comprising similarity scores for both the set $\cN_2$ of original images and the set $\cN_1$ of images that the attacker may add. We aim to choose a set $X \subseteq \cN_1$ of images such that adding the images of $\cN_1 \setminus X$ simultaneously minimizes the value every possible summarizing subset $Y$. The value of a summarizing set $Y$ is given by the following objective function.
\begin{align}\label{func:norm}
	f(X \dotcup Y)
	={} 
\sqrt[3]{\sum_{v\in \mathcal{N}\setminus X}  \sum_{u\in Y}M_{u,v}^3}-\frac{1}{|\mathcal{N}_2|}\sqrt[3]{\sum_{u\in Y}\sum_{v\in Y} M_{u,v}^3} + \lambda\cdot|X|\cdot\sqrt[3]{k}
\enspace.
\end{align}
Here, $M_{u,v}$ is the similarity score between images $u$ and $v$, which is assumed to be non-negative and symmetric (i.e., $M_{u, v} = M_{v, u} \geq 0$); and  $\lambda \in [0, 1]$ is a regularization parameter affecting the number of elements added by the adversary. Choosing a larger value for $\lambda$ results in a larger set $X$, and thus, less adversarial images being added. The objective function $f$ is jointly-submodular and non-negative (the proof is very similar to the proof that the function in Equation~\eqref{eq:robust_maxmin_objective} has these properties, and therefore, we omit it). Since we are interested in finding an attacker set $X$ that is good against the best summary set $Y$ of size $k$, the optimization problem that we aim to solve is
\[
	\min_{X\subseteq\mathcal{N}_1}\max_{\substack{Y \subseteq \cN_2\\|Y| \leq k}}f(X \dotcup Y)
	\enspace.
\]

Our experiments for this application are based on a subset of the CIFAR-$10$ data set~\cite{krizhevsky2009learning}. This subset includes $\numprint{10,000}$ tiny images belonging to $10$ classes. Each image consists of $32\times 32$ RGB pixels, and is thus, represented by a $\numprint{3,072}$ dimensional vector, and the cosine similarity method was used to compute similarities between images. In order to keep the running time computationally tractable, we randomly sampled from the data set in each experiment disjoint sets $\cN_1$ and $\cN_2$ of sizes $|\cN_1| = \numprint{2,000}$ and $|\cN_2|=250$.

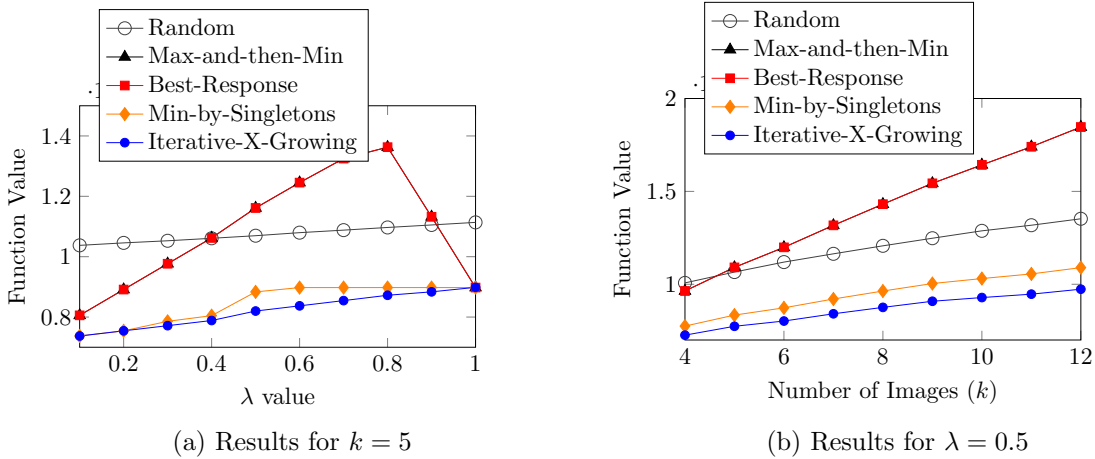
\begin{figure*}[tb]
\hspace*{\fill}
\begin{subfigure}[t]{0.47\textwidth}
 \begin{tikzpicture}[scale=0.82] 
\begin{axis}[
    xlabel = {$\lambda$ value},
    ylabel = {Function Value},
    xmin=0.1, xmax=1,
    ymin=7000, ymax=15000,
		legend cell align=left,
		legend style={at={(0.73,1.4)}},
		width=8cm,height=55mm
%		error bars/y dir=both,
    %error bars/y explicit
		]
		\pgfplotstableread{CSV_Image_Minmax/Image_k5_random_mean.csv}\random
		\pgfplotstablecreatecol[copy column from table={CSV_Image_Minmax/Image_k5_norm3_reg_STD.csv}{[index] 0}] {error} {\random}
		\addplot [name path = random, darkgray, mark = o, mark size=3pt] table [x expr=0.1+\coordindex/10, y index=0, y error index=1] {\random};
		\addlegendentry{Random}
		\pgfplotstableread{CSV_Image_Minmax/Image_k5_norm3_greed_mean.csv}\MaxOnly
		\pgfplotstablecreatecol[copy column from table={CSV_Image_Minmax/Image_k5_norm3_greed_STD.csv}{[index] 0}] {error} {\MaxOnly}
		\addplot [name path = theirs, black, mark = triangle*, mark size=3pt] table [x expr=0.1+\coordindex/10, y index=0, y error index = 1] {\MaxOnly};
		\addlegendentry{Max-and-then-Min}
		\pgfplotstableread{CSV_Image_Minmax/Image_k5_norm3_br_mean.csv}\BestResponse
		\pgfplotstablecreatecol[copy column from table={CSV_Image_Minmax/Image_k5_norm3_br_STD.csv}{[index] 0}] {error} {\BestResponse}
		\addplot [name path = theirs, red, mark = square*] table [x expr=0.1+\coordindex/10, y index=0, y error index = 1] {\BestResponse};
		\addlegendentry{Best-Response}
  		\pgfplotstableread{CSV_Image_Minmax/Image_k5_norm3_sing_mean.csv}\oursTwo
		\pgfplotstablecreatecol[copy column from table={CSV_Image_Minmax/Image_k5_norm3_sing_STD.csv}{[index] 0}] {error} {\oursTwo}
		\addplot [name path = our, orange, mark = diamond*, mark size=3pt] table [x expr=0.1+\coordindex/10, y index=0, y error index=1] {\oursTwo};
		\addlegendentry{\SingletonsAlg}
		\pgfplotstableread{CSV_Image_Minmax/Image_k5_norm3_reg_mean.csv}\ours
		\pgfplotstablecreatecol[copy column from table={CSV_Image_Minmax/Image_k5_norm3_reg_STD.csv}{[index] 0}] {error} {\ours}
		\addplot [name path = our, blue, mark = *] table [x expr=0.1+\coordindex/10, y index=0, y error index=1] {\ours};
		\addlegendentry{\IterativeAlg}
	\end{axis}\end{tikzpicture}\caption{Results for $k=5$}
    \label{img:image1}
 \end{subfigure}\hfill
 \begin{subfigure}[t]{0.47\textwidth}
  \begin{tikzpicture}[scale=0.82] 
\begin{axis}[
    xlabel = {Number of Images $(k)$},
    ylabel = {Function Value},
    xmin=4, xmax=12,
    ymin=7000, ymax=20000,
		legend cell align=left,
		legend style={at={(0.73,1.4)}},
		width=8cm,height=55mm
		%error bars/y dir=both,
    %error bars/y explicit
		]
		\pgfplotstableread{CSV_Image_Minmax/Image_gamma5_random_mean.csv}\random
  		\pgfplotstablecreatecol[copy column from table={CSV_Image_Minmax/Image_gamma5_k14_reg_STD.csv}{[index] 0}] {error} {\random}
		\addplot [name path = random, darkgray, mark = o, mark size=3pt] table [x expr=4+\coordindex, y index=0, y error index=1] {\random};
		\addlegendentry{Random}
		\pgfplotstableread{CSV_Image_Minmax/Image_gamma5_k14_greed_mean.csv}\MaxOnly
		\pgfplotstablecreatecol[copy column from table={CSV_Image_Minmax/Image_gamma5_k14_greed_STD.csv}{[index] 0}] {error} {\MaxOnly}
		\addplot [name path = theirs, black, mark = triangle*, mark size = 3pt] table [x expr=4+\coordindex, y index=0, y error index = 1] {\MaxOnly};
		\addlegendentry{Max-and-then-Min}
		\pgfplotstableread{CSV_Image_Minmax/Image_gamma5_k14_br_mean.csv}\BestResponse
		\pgfplotstablecreatecol[copy column from table={CSV_Image_Minmax/Image_gamma5_k14_br_STD.csv}{[index] 0}] {error} {\BestResponse}
		\addplot [name path = theirs, red, mark = square*] table [x expr=4+\coordindex, y index=0, y error index = 1] {\BestResponse};
		\addlegendentry{Best-Response}
  	\pgfplotstableread{CSV_Image_Minmax/Image_gamma5_k14_sing_mean.csv}\oursTwo
		\pgfplotstablecreatecol[copy column from table={CSV_Image_Minmax/Image_gamma5_k14_sing_STD.csv}{[index] 0}] {error} {\oursTwo}
		\addplot [name path = our, orange, mark = diamond*, mark size=3pt] table [x expr=4+\coordindex, y index=0, y error index=1] {\oursTwo};
		\addlegendentry{\SingletonsAlg}
    \pgfplotstableread{CSV_Image_Minmax/Image_gamma5_k14_reg_mean.csv}\ours
		\pgfplotstablecreatecol[copy column from table={CSV_Image_Minmax/Image_gamma5_k14_reg_STD.csv}{[index] 0}] {error} {\ours}
		\addplot [name path = our, blue, mark = *] table [x expr=4+\coordindex, y index=0, y error index=1] {\ours};
		\addlegendentry{\IterativeAlg}
  
	\end{axis}\end{tikzpicture}
  \caption{Results for $\lambda=0.5$}\label{img:image2}
  \end{subfigure}\hspace*{\fill}
	\caption{Empirical results for adversarial attack on image summarization. Both plots compares the performance of our algorithms {\SingletonsAlg} and {\IterativeAlg} with $3$ benchmarks for different value of the regularization parameter $\lambda$ and the cardinality parameter $k$.} \label{fig:image_summarization}
\end{figure*}

In our experiments, we study the change in the quality of the summaries obtained by the various algorithms and benchmarks as a function of the allowed number $k$ of images and the regularization parameter $\lambda$. Figure~\ref{img:image1} presents the outputs of our algorithms {\SingletonsAlg} and {\IterativeAlg} (with $\beta=0.2$) and three benchmarks for $k = 5$ and a varying regularization parameter $\lambda$. Figure~\ref{img:image2} presents the outputs of the same algorithms and benchmarks for $\lambda = 0.5$ and a varying limitation $k$ on the number of images in the summary. One can observe that both of our algorithms consistently outperform the benchmarks of Best-Response, Max-and-then-Min and Random, with the more involved algorithm {\IterativeAlg} tending to do better than the simpler algorithm {\SingletonsAlg}. Both figures are based on averaging $400$ executions of the algorithms, leading to a standard error of the mean of less than $10$ for all data points. It is also worth noting that the basic scarecrow benchmark ``Random'' outperforms the Best-Response and Max-and-then-Min benchmarks in many cases. This hints that the last heuristics are unreliable despite being natural, and highlights the significance of the methods we propose.

\section{Conclusion and future work}

In this paper we have initiated the systematical study of minimax optimization for combinatorial (discrete) settings with large domains. We have developed theoretical results fully mapping the approximability of max-min submodular optimization, and also obtained some understanding of the approximability of min-max submodular optimization. The above theoretical work has been complemented with empirical experiments demonstrating the value of our technique for the machine-learning tasks of efficient prompt engineering, ride-share difficulty kernelization, adversarial attacks on image summarization, and robust ride-share optimization.

We hope future work will lead to a fuller understanding of minimax submodular optimization, and will also consider classes of discrete functions beyond submodularity. A natural class to consider in that regard is the class of weakly-submodular functions~\cite{das2011submodular}, which extends the class of submodular functions. However, minimax optimization of this class seems to be difficult because, as far as we know, no algorithm is currently known even for plain minimization of weakly-submodular functions. 
Another open problem is to prove a performance guarantee for the Best-Response method.

% \ee{This paragraph should be edited.}
% For the application of prompt engineering, we have shown that robust discrete optimization can lead to improved performance on dialog state tracking when compared to non-robust top-$k$ retrieval given the same embedding space. In the future, we hope that this insight will lead to designing better prompting methods for other NLP tasks. It is also possible to develop additional perturbation models and embedding spaces that are well-suited for both robust retrieval and downstream application performance. 
% % designing the selection parmeters

\bibliographystyle{plain}
\bibliography{SubmodularMinMax}

\appendix
\section{Benchmarks and algorithm implementations} \label{app:implementation}

In this section we define all the benchmarks that we compare in Section~\ref{sec:exp}  against our algorithms. We then discuss the implementation details of these benchmarks and our algorithms.
\begin{itemize}
	\item Random: Returns a random feasible solution. In the max-min setting this means random $k$ elements from the ground set $\cN_1$, and in the min-max setting this means a random subset of $\cN_2$.
	\item Max-Only: This benchmark makes sense only in the max-min setting. It uses a submodular maximization algorithm to find a feasible set $Y$ (approximately) maximizing the objective for $X =\varnothing$.
	\item Max-and-then-Min: A variant of Max-Only for use in the min-max setting. It returns a set $X$ minimizing the objective given the set $Y$ chosen by Max-Only. Note that this is essentially equivalent to a single iteration of Best-Response.
	\item Top-$k$: This benchmark makes sense only in the max-min setting. It returns the $k$ singletons from $\cN_2$ with the maximum value, where the value of every singleton $u \in \cN_2$ is defined as $\min_{X \subseteq \cN_1} f(X \dotcup \{u\})$.
  \item Best-Response: This benchmark proceeds in iterations. In the first iteration, one obtains a subset $Y \in \cF_2$ (approximately) maximizing $f(Y)$ through the execution of a maximization algorithm, which is followed by finding a set $X \subseteq \cN_2$ minimizing $f(X \dotcup Y)$ by running a minimization algorithm. Subsequent iterations are similar to the first iteration, except that the set $Y$ chosen in these iterations is a set that (approximately) maximizes $f(X \dotcup Y)$, where $X$ is the minimizing set chosen in the previous iteration. The output is then the last set $Y$ in the max-min setting, and the last set $X$ in the min-max setting.
\end{itemize}

In most of our applications, we aim to optimize objectives that are not $\cN_2$-monotone, which requires a procedure for (approximate) maximization of non-monotone submodular functions. As mentioned in Section~\ref{ssc:related_work}, the state-of-the-art approximation guarantee for the case in which the objective function $f$ is not guaranteed to be monotone is currently $0.385$~\cite{buchbinder2019constrained}. However, the algorithm obtaining this approximation ratio is quite involved, which limits its practicality. Arguably, the state-of-the-art approximation ratio obtained by a ``simple" algorithm is $1/e$-approximation obtained by an algorithm called Random Greedy~\cite{buchbinder2014submodular}. In practice, the performance of this algorithm is comparable to that of the standard greedy algorithm, despite the last algorithm not having any approximation guarantee for non-monotone objective functions. Hence, throughout the experiments, the maximization component used in all the relevant benchmarks and algorithms is either the standard greedy algorithm or an accelerated version of it (suggested by~\cite{badanidiyuru2014fast}) named Threshold Greedy.

In our experiment we often report the values of the objective function corresponding to the output sets produced by the various benchmarks and algorithms. In the max-min setting, given an output set $X$, computing the objective value is done by utilizing an efficient minimizing algorithm to identify a minimizing set $X$. In the min-max setting, the situation is more involved as calculating the true objective value for given an output set $X$ cannot be done efficiently in sub-exponential time (as it corresponds to maximizing a submodular function subject to a cardinality constraint). Therefore, we use Threshold Greedy algorithm mentioned above to find a set $Y$ that approximately maximize the objective with respect to $X$, and then report the value corresponding to this set $Y$ as a proxy for the true objective value.

Our experiments for the min-max setting use a slightly modified version of {\IterativeAlg} (Algorithm~\ref{alg:square_root_approximation}). Specifically, we make two modifications to the algorithm.
\begin{itemize}
	\item {\IterativeAlg} grows a solution in iterations. As written, it outputs the set obtained after the last iteration. However, we chose to output instead the best set obtained after any number of iterations. This is a standard modification often used when applying to practice an iterative theoretical algorithm.
	\item Line~\ref{line:prime_selection} of {\IterativeAlg} looks for a set $X'_i$ that minimizes an expression involving two terms. The first of these terms $\sqrt{n_1} \cdot f(X \cup X_{i - 1})$ has the large coefficient $\sqrt{n_1}$. The value of this coefficient was chosen to fit the largest number of possible iterations that the algorithm may perform ($n_1 + 1$). However, in practice we found that the algorithm usually makes very few iterations. Thus, the use of the large coefficient $\sqrt{n_1}$ becomes sub-optimal. To truly show the empirical performance of {\IterativeAlg}, we replaced the coefficient $\sqrt{n_1}$ with a parameter $\beta$ whose value is chosen based on the application in question.
\end{itemize}

\end{document}